\documentclass[journal]{IEEEtran}
\usepackage{blindtext}
\usepackage{graphicx}
\usepackage{hyperref}
\usepackage{url}

\usepackage[utf8]{inputenc} 
\usepackage[T1]{fontenc}    
\usepackage{hyperref}       
\usepackage{url}            
\usepackage{booktabs}       
\usepackage{amsfonts}       
\usepackage{amsmath}
\usepackage{ amssymb }
\usepackage{amsthm}
\usepackage{nicefrac}       
\usepackage{microtype}      
\usepackage[ruled, linesnumbered]{algorithm2e}
\usepackage{subfig}
\usepackage{graphicx}
\usepackage[table]{xcolor}
\usepackage{float}
\usepackage{ulem}
\usepackage{tabularx}
\usepackage{array}
\newcolumntype{L}[1]{>{\raggedright\let\newline\\\arraybackslash\hspace{0pt}}m{#1}}
\newcolumntype{C}[1]{>{\centering\let\newline\\\arraybackslash\hspace{0pt}}m{#1}}
\newcolumntype{R}[1]{>{\raggedleft\let\newline\\\arraybackslash\hspace{0pt}}m{#1}}
\usepackage{multirow}
\usepackage{makecell}
\usepackage{wrapfig}
\usepackage{booktabs}
\usepackage{array}
\usepackage{colortbl}
\usepackage{footnote}
\usepackage{bbm}
\usepackage{lipsum}
\usepackage{tikz}
\newcommand*\circled[1]{\tikz[baseline=(char.base)]{
            \node[shape=circle,draw,inner sep=0pt] (char) {#1};}}

\newtheorem{theorem}{Theorem}[]
\newtheorem{corollary}{Corollary}[]
\newtheorem{lemma}[theorem]{Lemma}
\newtheorem{assumption}{Assumption}
\newtheorem{remark}{Remark}
\theoremstyle{definition}
\newtheorem{definition}{Definition}[]

\begin{document}
\normalem

\title{Clustered Federated Learning: Model-Agnostic Distributed Multi-Task Optimization under Privacy Constraints }
\author{Felix Sattler, Klaus-Robert M{\"u}ller*,~\IEEEmembership{Member,~IEEE}, and Wojciech Samek*,~\IEEEmembership{Member,~IEEE}
\thanks{This work was supported by the German Ministry for Education and Research as Berlin Big Data Center (01IS14013A) and the Berlin Center for Machine Learning (01IS18037I). Partial funding by DFG is acknowledged (EXC 2046/1, project-ID: 390685689). This
work was also supported by the Information \& Communications Technology Planning \& Evaluation (IITP) grant funded by the Korea government (No. 2017-0-00451).}
\thanks{F. Sattler and W. Samek are with Fraunhofer Heinrich Hertz Institute, 10587 Berlin, Germany (e-mail: wojciech.samek@hhi.fraunhofer.de).}
\thanks{K.-R. M{\"u}ller is with the Technische Universit{\"a}t Berlin, 10587 Berlin, Germany, with the Max Planck Institute for Informatics, 66123 Saarbr{\"u}cken, Germany, and also with the Department of Brain and Cognitive Engineering, Korea University, Seoul 136-713, South Korea (e-mail: klaus-robert.mueller@tu-berlin.de).}}
\markboth{Sattler et al.\ -- Clustered Federated Learning: Model-Agnostic Distributed Multi-Task Optimization under Privacy Constraints }%
{Sattler et al.\ -- Clustered Federated Learning: Model-Agnostic Distributed Multi-Task Optimization under Privacy Constraints }

\maketitle

\begin{abstract}
Federated Learning (FL) is currently the most widely adopted framework for collaborative training of (deep) machine learning models under privacy constraints. 
Albeit it's popularity, it has been observed 
that Federated Learning yields suboptimal results if the local clients' data distributions diverge.
To address this issue, we present Clustered Federated Learning (CFL), a novel Federated Multi-Task Learning (FMTL) framework, which exploits geometric properties of the FL loss surface, to group the client population into clusters with jointly trainable data distributions.
In contrast to existing FMTL approaches, CFL does not require any modifications to the FL communication protocol to be made, is applicable to general non-convex objectives (in particular deep neural networks) and comes with strong mathematical guarantees on the clustering quality. CFL is flexible enough to handle client populations that vary over time and can be implemented in a privacy preserving way. As clustering is only performed after Federated Learning has converged to a stationary point, CFL can be viewed as a post-processing method that will always achieve greater or equal performance than conventional FL by allowing clients to arrive at more specialized models. We verify our theoretical analysis in experiments with deep convolutional and recurrent neural networks on commonly used Federated Learning datasets.  
\end{abstract}

\section{Introduction}
\label{sec:intro}
Federated Learning \cite{mcmahan2016communication}\cite{konevcny2016federated}\cite{bonawitz2017practical}\cite{bonawitz2019towards}\cite{li2019federated} is a distributed training framework, which allows multiple clients (typically mobile or IoT devices) to jointly train a single deep learning model on their combined data in a communication-efficient way, without requiring any of the participants to reveal their private training data to a centralized entity or to each other.
Federated Learning realizes this goal via an iterative three-step protocol where in every communication round $t$, the clients first synchronize with the server by downloading the latest master model $\theta_t$. Every client then proceeds to improve the downloaded model, by performing multiple iterations of stochastic gradient descent with mini-batches sampled from it's local data $D_i$, resulting in a weight-update vector
\begin{align}
\label{eq:sgd}
\Delta\theta_i^{t+1}=\text{SGD}_k(\theta^t,D_i)-\theta^t,~i=1,..,m
\end{align}
Finally, all clients upload their computed weight-updates to the server, where they are aggregated by weighted averaging according to
\begin{align}
\label{eq:update}
\theta^{t+1}=\theta^t+\sum_{i=1}^m\frac{|D_i|}{|D|}\Delta\theta_i^{t+1}
\end{align}  
to create the next master model. The procedure is summarized in Algorithm \ref{alg:FL}.

Federated Learning implicitly makes the assumption that it is possible for one single model to fit all client's data generating  distributions $\varphi_i$ at the same time. Given a model $f_\theta : \mathcal{X} \rightarrow \mathcal{Y}$ parametrized by $\theta\in\Theta$ and a loss function $l : \mathcal{Y}\times\mathcal{Y}\rightarrow\mathbb{R}_{\geq0}$ we can formally state this assumption as follows:

\begin{assumption}
\label{ass:1}
\textbf{("Federated Learning"):} 
\textit{
There exists a parameter configuration $\theta^*\in\Theta$, that (locally) minimizes the risk on all clients' data generating distributions at the same time: 
\begin{align}
\label{eq:fl_assume}
R_i(\theta^*)\leq R_i(\theta)~\forall \theta\in B_\varepsilon(\theta^*), i=1,..,m
\end{align} 
}
for some $\varepsilon>0$. Hereby
\begin{align}
R_i(\theta) = \int l(f_\theta(x), y)d\varphi_i(x,y)
\end{align}
is the risk function associated with distribution $\varphi_i$. 
\end{assumption}

It is easy to see that this assumption is not always satisfied. Concretely it is violated if either (a) clients have disagreeing conditional distributions $\varphi_i(y|x)\neq \varphi_j(y|x)$ or (b) the model $f_\theta$ is not expressive enough to fit all distributions at the same time. Simple counter examples for both cases are presented in Figure \ref{fig:ex_incongruent}.  

\begin{figure}
\centering
\includegraphics[width=0.5\textwidth]{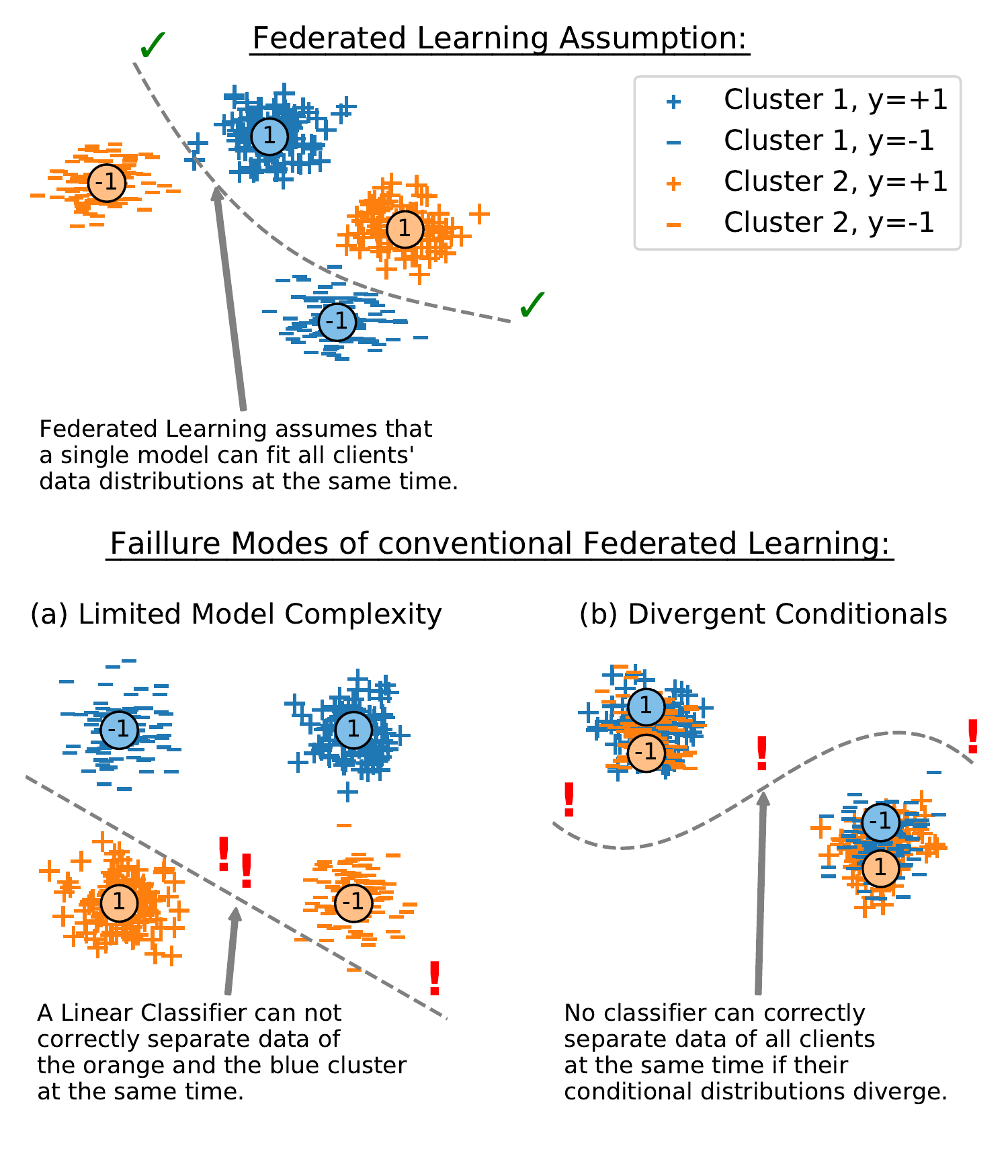}
\caption{Two toy cases in which the Federated Learning Assumption is violated. Blue points belong to clients from the first cluster while orange points belong to clients from the second cluster. Left: Federated XOR-problem. An insufficiently complex model is not capable of fitting all clients' data distributions at the same time. Right: If different clients' conditional distributions diverge, no single model can fit all distributions at the same time. In both cases the data on clients belonging to the same cluster can be easily separated.}
\label{fig:ex_incongruent}
\end{figure}

In the following we will call a set of clients and their data generating distributions $\varphi$ \emph{congruent} (with respect to $f$ and $l$) if they satisfy Assumption \ref{ass:1} and \emph{incongruent} if they don't.

In this work, we argue that Assumption \ref{ass:1} is frequently violated in real Federated Learning applications, especially given the fact that in Federated Learning clients (a) can hold arbitrary non-iid data, which can not be audited by the centralized server due to privacy constraints and (b) typically run on limited hardware which puts restrictions on the model complexity.  For illustration consider the following practical scenarios:

\emph{Varying Preferences:} Assume a scenario where every client holds a local dataset of images of human faces and the goal is to train an 'attractiveness' classifier on the joint data of all clients. Naturally, different clients will have varying opinions about the attractiveness of certain individuals, which corresponds to disagreeing conditional distributions on all clients' data. Assume for instance that one half of the client population thinks that people wearing glasses are attractive, while the other half thinks that those people are unattractive. In this situation one single model will \emph{never} be able to accurately predict attractiveness of glasses-wearing people for all clients at the same time (confer also Figure \ref{fig:ex_incongruent} right).

\emph{Limited Model Complexity:} Assume a number of clients are trying to jointly train a language model for next-word prediction on private text messages. In this scenario the statistics of a client's text messages will likely vary a lot based on demographic factors, interests, etc. For instance, text messages composed by teenagers will typically exhibit different statistics than than those composed by elderly people. An insufficiently expressive model will not be able to fit the data of all clients at the same time.

\emph{Presence of Adversaries:} A special case of incongruence is given, if a subset of the client population behaves in an adversarial manner. In this scenario the adversaries could deliberately alter their local data distribution in order to encode arbitrary behavior into the jointly trained model, thus affecting the model decisions on all other clients and causing potential harm.

\ \\
The goal in Federated Multi-Task Learning is to provide every client with a model that optimally fits it's local data distribution.
In all of the above described situations the ordinary Federated Learning framework, in which all clients are treated equally and only one single global model is learned, is not capable of achieving this goal. 

In order to incorporate the above presented problems with incongruent data generating distributions,
we suggest to generalize the conventional Federated Learning Assumption: 

\begin{assumption}
\label{ass:2}
\textbf{("Clustered Federated Learning"):} 
\textit{
There exists a partitioning $\mathcal{C}=\{c_1,..,c_k\}$, $\bigcup_{i=1}^k c_k = \{1,..,m\}$ of the client population, such that every subset of clients $c\in\mathcal{C}$ satisfies the conventional Federated Learning Assumption. 
}
\end{assumption}
\ \\

The remainder of this manuscript is organized as follows: In the next section (\ref{sec:theory}) we will derive a computationally efficient tool based on the cosine similarity between the clients' gradient updates that \emph{provably} allows us to infer whether two members of the client population have the same data generating distribution, thus making it possible for us to infer the clustering structure $\mathcal{C}$. Based on the theoretical insights in section \ref{sec:theory} we present the Clustered Federated Learning Algorithm in section \ref{sec:cfl}.  
After reviewing related literature in section \ref{sec:litterature}, we address implementation details in section \ref{sec:implementation} and demonstrate that our novel method can be implemented without making modifications to the Federated Learning communication protocol (section \ref{sec:weight-updates}). We furthermore show that our method can be implemented in a privacy preserving way (section \ref{sec:privacy}) and is flexible enough to handle client populations that vary over time (section \ref{sec:tree}). Finally, in section \ref{sec:experiments}, we perform extensive experiments on a variety of convolutional and recurrent neural networks applied to common Federated Learning datasets.

\section{Cosine Similarity based Clustering}
\label{sec:theory}
In this paper, we address the question of how to solve distributed learning problems that satisfy Assumption \ref{ass:2} (which generalizes the Federated Learning Assumption \ref{ass:1}). This will require us to first identify the correct partitioning $\mathcal{C}$, which at first glance seems like a daunting task, as under the Federated Learning paradigm the server has no access to the clients' data, their data generating distributions or any meta information thereof.

An easier task than trying to directly infer the entire clustering structure $\mathcal{C}$, is to find a correct bi-partitioning in the sense of the following definition:
\begin{definition}
\label{def:1}
Let $m\geq k\geq 2$ and 
\begin{align}
I : \{1,..,m\} \rightarrow \{1,..,k\}, i\mapsto I(i)
\end{align}
be the mapping that assigns a client $i$ to it's data generating distribution $\varphi_{I(i)}$. Then we call a bi-partitioning $c_1\dot{\cup} c_2=\{1,..,m\}$ with $c_1\neq \emptyset$ and $c_2\neq \emptyset$ \emph{correct} if and only if
\begin{align}
I(i)\neq I(j)~~ \forall i\in c_1, j\in c_2.
\end{align} 
\end{definition}
In other words, a bi-partitioning is called \emph{correct}, if clients with the same data generating distribution end up in the same cluster. It is easy to see, that the clustering $\mathcal{C}=\{c_1,..,c_k\}$ can be obtained after exactly $k-1$ correct bi-partitions.

In the following we will demonstrate that there exists an explicit criterion based on which a correct bi-partitioning can be inferred. To see this, let us first look at the following simplified Federated Learning setting with $m$ clients, in which the data on every client was sampled from one of two data generating distributions $\varphi_1$, $\varphi_2$ such that
\begin{align}
D_i\sim\varphi_{I(i)}(x,y).
\end{align}
Every Client is associated with an empirical risk function
\begin{align}
r_i(\theta) = \sum_{x\in D_i} l_\theta(f(x_i),y_i)
\end{align}
which approximates the true risk arbitrarily well if the number of data points on every client is sufficiently large
\begin{align}
\label{eq:approx}
r_i(\theta) \approx R_{I(i)}(\theta) := \int_{x,y} l_\theta(f(x),y) d\varphi_{I(i)}(x,y).
\end{align}
For demonstration purposes let us first assume equality in \eqref{eq:approx}. Then the Federated Learning objective becomes
\begin{align}
\label{eq:federated_obj}
F(\theta) := \sum_{i=1}^m \frac{|D_i|}{|D|} r_i(\theta) = a_1 R_1(\theta) + a_2 R_2(\theta)
\end{align}
with $a_1=\sum_{i, I(i)=1} |D_i|/|D|$ and $a_2=\sum_{i, I(i)=2} |D_i|/|D|$.
Under standard assumptions it has been shown \cite{li2019convergence}\cite{sahu2018convergence} that the Federated Learning optimization protocol described in equations \eqref{eq:sgd} and \eqref{eq:update} converges to a stationary point $\theta^*$ of the Federated Learning objective. In this point it holds that
\begin{align}
\label{eq:equilib}
0 = \nabla F(\theta^*) = a_1\nabla R_1(\theta^*) + a_2\nabla R_2(\theta^*)
\end{align} 
Now we are in one of two situations. Either it holds that $\nabla R_1(\theta^*) =\nabla R_2(\theta^*)=0$, in which case we have simultaneously minimized the risk of all clients. This means $\varphi_1$ and $\varphi_2$ are congruent and we have solved the distributed learning problem. Or, otherwise, it has to hold 
\begin{align}
\label{eq:opposite}
\nabla R_1(\theta^*) =- \frac{a_2}{a_1}\nabla R_2(\theta^*)\neq0
\end{align}
and $\varphi_1$ and $\varphi_2$ are incongruent. In this situation the \emph{cosine similarity} between the gradient updates of any two clients is given by
\begin{align}
\begin{split}
\alpha_{i,j}:=\alpha(\nabla r_i(\theta^*), \nabla r_j(\theta^*))&:=\frac{\langle\nabla r_i(\theta^*), \nabla r_j(\theta^*)\rangle}{\|\nabla r_i(\theta^*)\|\|\nabla r_j(\theta^*)\|}\\
&=\frac{\langle\nabla R_{I(i)}(\theta^*), \nabla R_{I(j)}(\theta^*)\rangle}{\|\nabla R_{I(i)}(\theta^*)\|\|\nabla R_{I(j)}(\theta^*)\|}\\
\label{eq:cossim}
&=\begin{cases}
1 & \text{ if }I(i)=I(j)\\
-1 & \text{ if }I(i)\neq I(j)
\end{cases}
\end{split}
\end{align}
Consequently, a correct bi-partitioning is given by
\begin{align}
c_1=\{i|\alpha_{i,0}=1\},~c_2=\{i|\alpha_{i,0}=-1\}.
\end{align}
This consideration provides us with the insight that, in a stationary solution of the Federated Learning objective $\theta^*$, we can \emph{distinguish clients based on their hidden data generating distribution by inspecting the cosine similarity between their gradient updates}. For a visual illustration of the result we refer to Figure \ref{fig:toy}. 


If we drop the equality assumption in \eqref{eq:approx} and allow for an arbitrary number of data generating distributions, we obtain the following generalized version of result \eqref{eq:cossim}:
\begin{figure}
\centering
\includegraphics[width=0.5\textwidth]{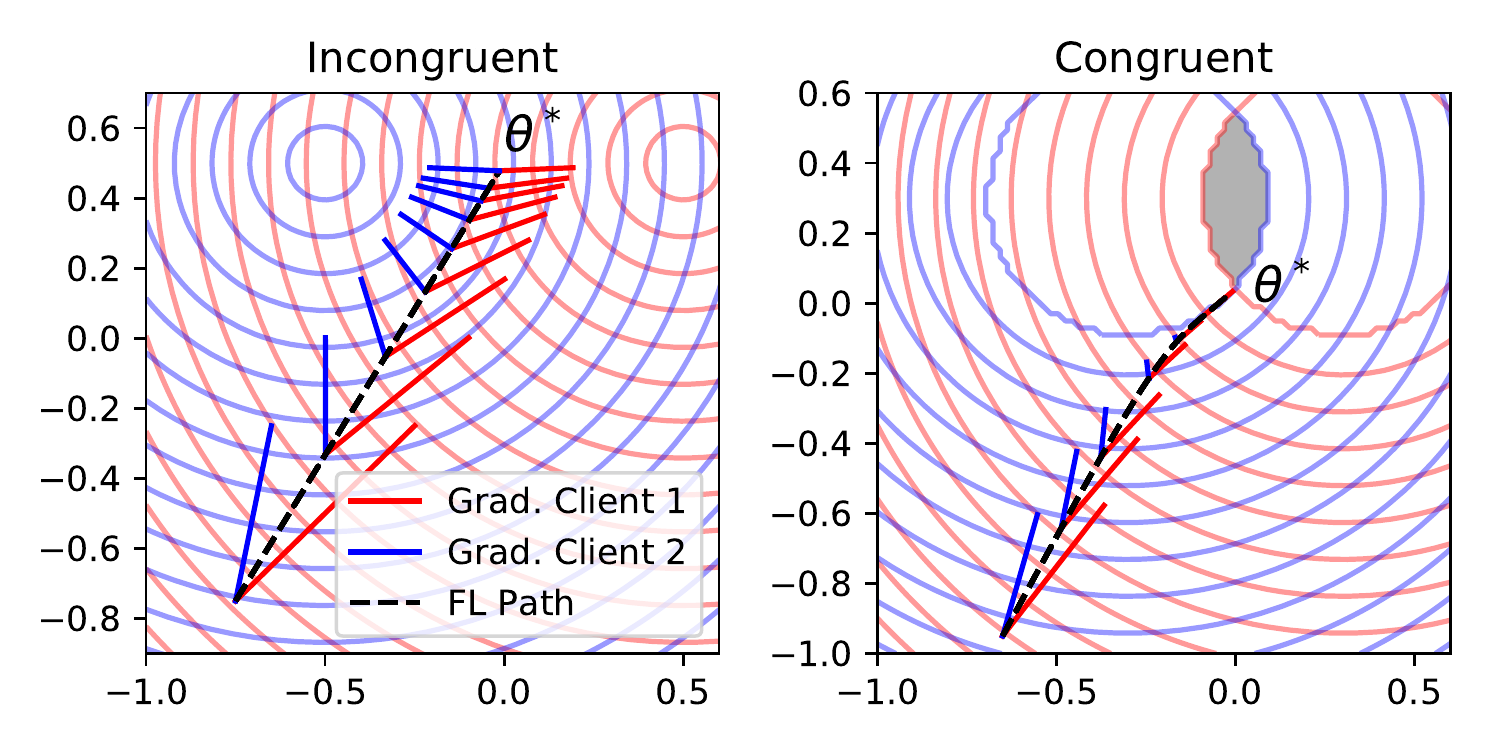}
\caption{Displayed are the optimization paths of Federated Learning with two clients, applied to two different toy problems with incongruent (left) and congruent (right) risk functions. In the incongruent case Federated Learning converges to a stationary point of the FL objective where the gradients of the two clients are of positive norm and point into opposite directions. In the congruent case there exists an area (marked grey in the plot) where both risk functions are minimized. If Federated Learning converges to this area the norm of both client's gradient updates goes to zero. By inspecting the gradient norms the two cases can be distinguished.}
\label{fig:toy}
\end{figure}

\begin{theorem}[Separation Theorem]
\label{theo:separation_theorem}
Let $D_1,..,D_m$ be the local training data of $m$ different clients, each dataset sampled from one of $k$ different data generating distributions $\varphi_1,..,\varphi_k$, such that $D_i\sim\varphi_{I(i)}(x,y)$. Let the empirical risk on every client approximate the true risk at every stationary solution of the Federated Learning objective $\theta^*$ s.t.
\begin{align}
\|\nabla R_{I(i)}(\theta^*)\| > \|\nabla R_{I(i)}(\theta^*)-\nabla r_i(\theta^*)\|
\end{align} 
and define 
\begin{align}
\gamma_i:=\frac{\|\nabla R_{I(i)}(\theta^*)-\nabla r_i(\theta^*)\|}{\|\nabla R_{I(i)}(\theta^*)\|}\in[0,1)
\end{align}
Then there exists a bi-partitioning $c_1^*\cup c_2^*=\{1,..,m\}$ of the client population such that that the maximum similarity between the updates from any two clients from different clusters can be bounded from above according to
\begin{align}
\label{eq:res1}
&\alpha_{cross}^{max} := \min_{c_1\cup c_2=\{1,..,m\}}\max_{i\in c_1,j\in c_2}\alpha(\nabla r_i(\theta^*), \nabla r_j(\theta^*))\\
&=\max_{i\in c_1^*,j\in c_2^*}\alpha(\nabla r_i(\theta^*), \nabla r_j(\theta^*))\\
&\leq
\begin{cases}
 \cos(\frac{\pi}{k-1})H_{i,j}+\sin(\frac{\pi}{k-1})\sqrt{1-H_{i,j}^2} & \text{ if }H\geq \cos(\frac{\pi}{k-1}) \\
 1 & \text{ else }
\end{cases}
\end{align}
with
\begin{align}
H_{i,j} = -\gamma_i\gamma_j+\sqrt{1-\gamma_i^2}\sqrt{1-\gamma_j^2}\in(-1,1].
\end{align}
At the same time the similarity between updates from clients which share the same data generating distribution can be bounded from below by
\begin{align}
\label{eq:res2}
\alpha_{intra}^{min} := \min_{\underset{I(i)=I(j)}{i,j}}\alpha(\nabla_\theta r_i(\theta^*), \nabla_\theta r_j(\theta^*))\geq \min_{\underset{I(i)=I(j)}{i,j}}H_{i,j}.
\end{align}
\end{theorem}
The proof of Theorem \ref{theo:separation_theorem} can be found in the Appendix. 
\begin{remark}
In the case with two data generating distributions ($k=2$) the result simplifies to
\begin{align}
\alpha_{cross}^{max}=\max_{i\in c_1^*,j\in c_2^*}\alpha(\nabla_\theta r_i(\theta^*), \nabla_\theta r_j(\theta^*))\leq \max_{i\in c_1^*,j\in c_2^*}-H_{i,j}
\end{align}
for a certain partitioning, respective
\begin{align}
\alpha_{intra}^{min}=\min_{\underset{I(i)=I(j)}{i,j}}\alpha(\nabla_\theta r_i(\theta^*), \nabla_\theta r_j(\theta^*))\geq \min_{\underset{I(i)=I(j)}{i,j}} H_{i,j}
\end{align}
for two clients from the same cluster. If additionally $\gamma_i=0$ for all $i=1,..,m$ then $H_{i,j}=1$ and we retain result \eqref{eq:cossim}.
\end{remark}
From Theorem \ref{theo:separation_theorem} we can directly deduce a correct separation rule:
\begin{corollary}
\label{corr:1}
If in Theorem \ref{theo:separation_theorem} $k$ and  $\gamma_i$, $i=1,..,m$ are in such a way that
\begin{align}
\alpha_{intra}^{min}>\alpha_{cross}^{max}
\end{align}
then the partitioning 
\begin{align}
\label{eq:split}
c_1,c_2\leftarrow \arg\min_{c_1\cup c_2 = c}(\max_{i\in c_1, j\in c_2} \alpha_{i,j}).
\end{align}
is \emph{always} correct in the sense of Definition \ref{def:1}.
\end{corollary}
\begin{proof}
Set \begin{align}
c_1,c_2\leftarrow \arg\min_{c_1\cup c_2 = c}(\max_{i\in c_1, j\in c_2} \alpha_{i,j})
\end{align}
and let $i\in c_1$, $j\in c_2$ then 
\begin{align}
\alpha_{i,j} \leq \alpha_{cross}^{max} <  \alpha_{intra}^{min} = \min_{\underset{I(i)=I(j)}{i,j}}\alpha_{i,j}
\end{align}
and hence $i$ and $j$ can not have the same data generating distribution.
\end{proof}
This consideration leads us to the notion of the separation gap:
\begin{definition}[Separation Gap]
\label{def:separationgap}
Given a cosine-similarity matrix $\alpha$ and a mapping from client to data generating distribution $I$ we define the notion of a separation gap
\begin{align}
g(\alpha)&:=\alpha_{intra}^{min}-\alpha_{cross}^{max}\\&=\min_{\underset{I(i)=I(j)}{i,j}}\alpha_{i,j}-\min_{c_1\cup c_2 = c}(\max_{i\in c_1, j\in c_2} \alpha_{i,j})
\end{align}
\end{definition}
\begin{remark}
By Corollary \ref{corr:1} the bi-partitioning \eqref{eq:split} will be correct in the sense of Definition \ref{def:1} if and only if the separation gap is greater than zero.
\end{remark}
\begin{figure}
\centering
\includegraphics[width=0.5\textwidth]{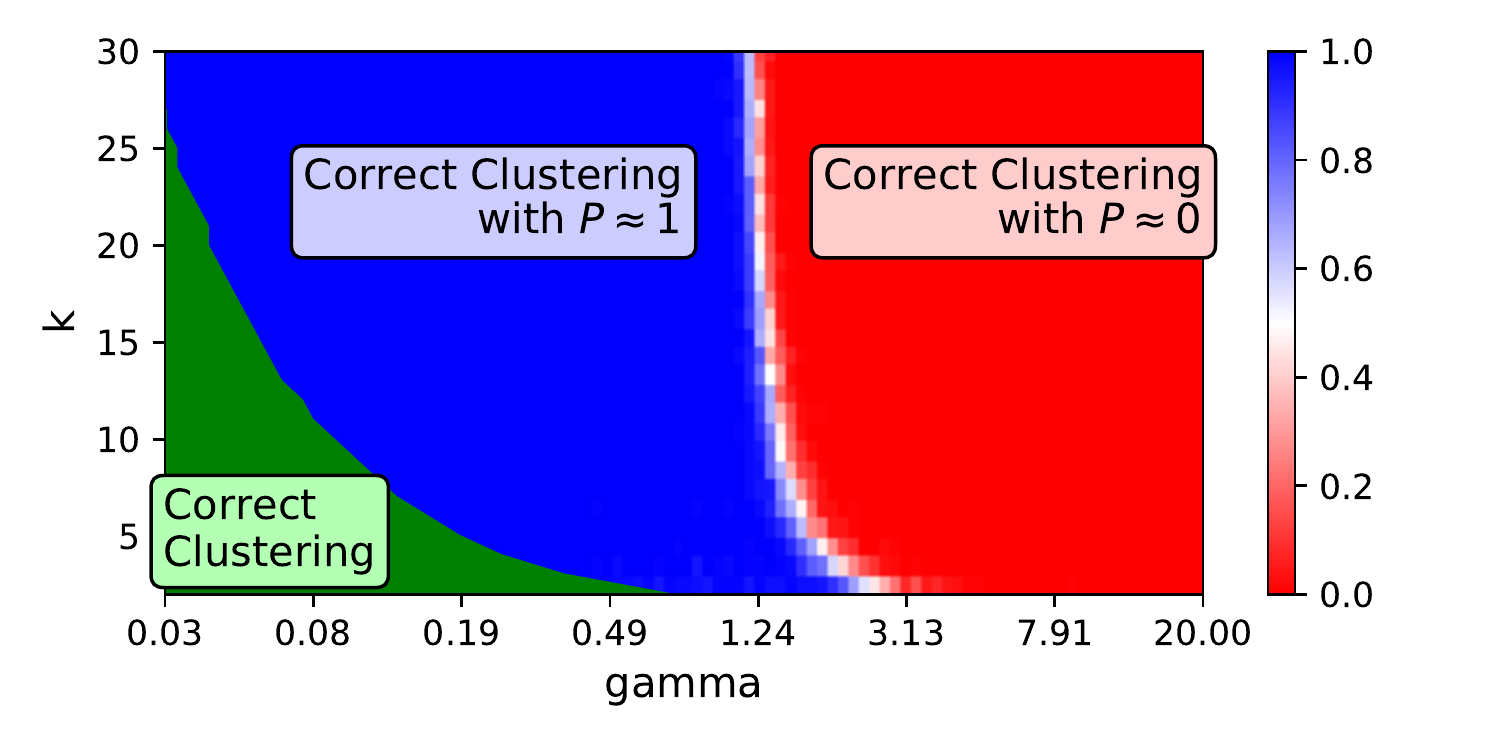}
\caption{Clustering quality as a function of the number of data generating distributions $k$ and the relative approximation noise $\gamma$. For all values of $k$ and $\gamma$ in the green area, CFL will \emph{always} correctly separate the clients (by Theorem \ref{theo:separation_theorem}). For all values of $k$ and $\gamma$ in the blue area we find empirically that CFL will correctly separate the clients with probability close to 1.}
\label{fig:parameters}
\end{figure}
Theorem \ref{theo:separation_theorem} gives an estimate for the similarities in the \emph{absolute worst-case}. In practice $\alpha_{intra}^{min}$ typically will be much larger and $\alpha_{cross}^{max}$ typically will be much smaller, especially if the parameter dimension $d$ is large.
For instance, if we set $d=10^2$ (which is still many orders of magnitude smaller than typical modern neural networks), $m=3k$, and assume $\nabla R_{I(i)}(\theta^*)$ and $\nabla R_{I(i)}(\theta^*)-\nabla r_i(\theta^*)$ to be normally distributed for all $i=1,..,m$ then experimentally we find (Figure \ref{fig:parameters}) that
\begin{align}
P[\text{"Correct Clustering"}]=P[g(\alpha)>0]\approx 1
\end{align}
even for large values of $k>10$ and 
\begin{align}
\gamma_{max}:=\max_{i=1,..,m}\gamma_i>1.
\end{align}
This means that using the cosine similarity criterion \eqref{eq:split} we can readily find a correct bi-partitioning $c_1, c_2$ even if the number of data generating distributions is high and the empirical risk on every client's data is only a very loose approximation of the true risk. 

\subsection{Distinguishing Congruent and Incongruent Clients}

In order to appropriately generalize the classical Federated Learning setting, we need to make sure to only split up clients with \emph{incongruent} data distributions. In the classical congruent non-iid Federated Learning setting described in \cite{mcmahan2016communication} where one single model can be learned, performance will typically degrade if clients with varying distributions are separated into different clusters due to the restricted knowledge transfer between clients in different clusters. Luckily we have a criterion at hand to distinguish the two cases. To realize this we have inspect the gradients computed by the clients at a stationary point $\theta^*$. When client distributions are incongruent, the stationary solution of the Federated Learning objective by definition can not be stationary for the individual clients. Hence the norm of the clients' gradients has to be strictly greater than zero. If conversely the client distributions are congruent, Federated optimization will be able to jointly optimize all clients' local risk functions and hence the norm of the clients' gradients will tend towards zero as we are approaching the stationary point. Based on this observation we can formulate the following criteria which allow us make the decision whether to split or not: Splitting should only take place if it holds that both (a) we are close to a stationary point of the FL objective
\begin{align}
\label{eq:servernorm}
0 \leq \|\sum_{i=1,..,m}\frac{D_i}{|D_c|}\nabla_\theta r_i(\theta^*)\|<\varepsilon_1
\end{align}     
and (b) the individual clients are far from a stationary point of their local empirical risk
\begin{align}
\label{eq:clientnorm}
\max_{i=1,..,m} \|\nabla_\theta r_i(\theta^*)\|>\varepsilon_2 > 0
\end{align}
Figure \ref{fig:toy} gives a visual illustration of this idea for a simple two dimensional problem. We will also experimentally verify the clustering criteria \eqref{eq:servernorm} and \eqref{eq:clientnorm} in section \ref{sec:criteria}.

In practice we have another viable option to distinguish the congruent from the incongruent case. As splitting will only be performed after Federated Learning has converged to a stationary point, we always have computed the conventional Federated Learning solution as part of Clustered Federated Learning. This means that if after splitting up the clients a degradation in model performance is observed, it is always possible to fall back to the Federated Learning solution. In this sense Clustered Federated Learning will always improve the Federated Learning performance (or perform equally well at worst).

\section{Clustered Federated Learning} 
\label{sec:cfl}
Clustered Federated Learning recursively bi-partitions the client population in a top-down way: Starting from an initial set of clients $c=\{1,..,m\}$ and a parameter initialization $\theta_0$, CFL performs Federated Learning according to Algorithm \ref{alg:FL}, in order to obtain a stationary solution $\theta^*$ of the FL objective. After Federated Learning has converged, the stopping criterion 
\begin{align}
\label{eq:stopp}
0 \leq  \max_{i\in c} \|\nabla_\theta r_i(\theta^*)\|<\varepsilon_2 
\end{align}
is evaluated. If criterion \eqref{eq:stopp} is satisfied, we know that all clients are sufficiently close to a stationary solution of their local risk and consequently CFL terminates, returning the FL solution $\theta^*$. If on the other hand, criterion \eqref{eq:stopp} is violated, this means that the clients are incongruent and the server computes the pairwise cosine similarities $\alpha$ between the clients' latest transmitted updates according to equation \eqref{eq:cossim}. Next, the server separates the clients into two clusters in such a way that the maximum similarity between clients from different clusters is minimized
\begin{align}
c_1,c_2\leftarrow \arg\min_{c_1\cup c_2 = c}(\max_{i\in c_1, j\in c_2} \alpha_{i,j}).
\end{align}
This optimal bi-partitioning problem at the core of CFL can be solved in $\mathcal{O}(m^3)$ using Algorithm \ref{alg:bipartition}. Since in Federated Learning it is assumed that the server has far greater computational power than the clients, the overhead of clustering will typically be negligible.

As derived in section \ref{sec:theory}, a correct bi-partitioning can always be ensured if it holds that 
\begin{align}
\alpha_{intra}^{min}>\alpha_{cross}^{max}.
\end{align}
While the optimal cross-cluster similarity $\alpha^{max}_{cross}$ can be easily computed in practice, computation of the intra cluster similarity requires knowledge of the clustering structure and hence $\alpha_{intra}^{min}$ can only be estimated using Theorem \ref{theo:separation_theorem} according to
\begin{align}
\alpha_{intra}^{min}&\geq \min_{\underset{I(i)=I(j)}{i,j}}-\gamma_i\gamma_j+\sqrt{1-\gamma_i^2}\sqrt{1-\gamma_j^2}\\&\geq  1 - 2\gamma_{max}^2.
\end{align}
Consequently we know that the bi-partitioning will be correct if
\begin{align}
\label{eq:crit}
\gamma_{max} < \sqrt{\frac{1-\alpha_{cross}^{max}}{2}}.
\end{align} 
independent of the number of data generating distributions $k$!
This criterion allows us to reject bi-partitionings, based on our assumptions on the approximation noise $\gamma_{max}$ (which is an interpretable hyperparameter).

If criterion \eqref{eq:crit} is satisfied, CFL is recursively re-applied to each of the two separate groups starting from the stationary solution $\theta^*$. Splitting recursively continues on until (after at most $k-1$ recursions) none of the sub-clusters violate the stopping criterion anymore, at which point all groups of mutually congruent clients $\mathcal{C}=\{c_1,..,c_k\}$ have been identified, and the Clustered Federated Learning problem characterized by Assumption \ref{ass:2} is solved. The entire recursive procedure is presented in Algorithm \ref{alg:CFL}.

%

\begin{algorithm} 
\caption{Optimal Bipartition}
\label{alg:bipartition}
\DontPrintSemicolon
\textbf{input:} Similarity Matrix $\alpha\in[-1,1]^{m,m}$\\
\textbf{outout:} bi-partitioning $c_1, c_2$ satisfying \eqref{eq:split}\\
\textbullet~ $s \leftarrow \text{argsort}(-\alpha[:])\in\mathbb{N}^{m^2}$\\
\textbullet~ $\mathcal{C}\leftarrow \{\{i\}|i=1,..,m\}$\\
\For{$i=1,..,m^2$}{
\textbullet~ $i_1\leftarrow s_i \text{ div } m$; $i_2\leftarrow s_i \text{ mod } m$\\
\textbullet~ $c_{tmp} \leftarrow \{\}$\\
\For{$c \in \mathcal{C}$}{
\If{$i_1 \in c$ \textbf{or} $i_2 \in c$}{
\textbullet~ $c_{tmp} \leftarrow c_{tmp} \cup c$\\
\textbullet~ $\mathcal{C} \leftarrow \mathcal{C} \setminus c$\\
}
}
\textbullet~ $\mathcal{C}\leftarrow \mathcal{C}\cup \{c_{tmp}\}$ \\
\If{$|\mathcal{C}|=2$}{
\Return $\mathcal{C}$
}
}
\end{algorithm}

\begin{algorithm}
\caption{Federated Learning (FL)}
\label{alg:FL}
\DontPrintSemicolon
\textbf{Input:} initial parameters $\theta$, set of clients $c$, $\varepsilon_1>0$\\
\Repeat{$\|\sum_{i\in c}\frac{|D_i|}{|D_c|}\Delta \theta_i\|<\varepsilon_1$}{
\For{$i\in c$ \textbf{in parallel}}{
\textbullet~ $\theta_i \leftarrow \theta$\\
\textbullet~ $\Delta \theta_i \leftarrow \text{SGD}_n(\theta_i, D_i)-\theta_i$ \\
}
\textbullet~ $\theta \leftarrow \theta+\sum_{i\in c}\frac{|D_i|}{|D_c|}\Delta \theta_i$ \\
}
\Return $\theta$
\end{algorithm}

\begin{algorithm}
\caption{Clustered Federated Learning (CFL)}
\label{alg:CFL}
\DontPrintSemicolon
\textbf{Input:} initial parameters $\theta$, set of clients $c$, $\gamma_{max}\in[0,1]$, $\varepsilon_2>0$\\
\textbullet~$\theta^*\leftarrow \text{FederatedLearning}(\theta,c)$\\
\textbullet~$\alpha_{i,j}\leftarrow \frac{\langle\nabla r_i(\theta^*), \nabla r_j(\theta^*)\rangle}{\|\nabla r_i(\theta^*)\|\|\nabla r_j(\theta^*)\|}, i,j\in c$\\
\textbullet~$c_1,c_2\leftarrow \arg\min_{c_1\cup c_2 = c}(\max_{i\in c_1, j\in c_2} \alpha_{i,j})$\\
\textbullet~$\alpha_{cross}^{max}\leftarrow \max_{i\in c_1, j\in c_2} \alpha_{i,j}$\\
\uIf{$\max_{i\in c} \|\nabla r_i(\theta^*)\|\geq \varepsilon_2$ \textbf{and} $\sqrt{\frac{1-\alpha_{cross}^{max}}{2}}>\gamma_{max}$}{
\textbullet~$\theta_i^*, i\in c_1\leftarrow \text{ClusteredFederatedLearning}(\theta^*, c_1)$ \\
\textbullet~$\theta_i^*, i\in c_2\leftarrow \text{ClusteredFederatedLearning}(\theta^*, c_2)$ \\
}
\Else{
\textbullet~$\theta_i^*\leftarrow \theta^*$, $i\in c$
}
\Return $\theta_i^*, i\in c$
\end{algorithm}

\section{Related Work}
\label{sec:litterature}
Federated Learning \cite{mcmahan2016communication}\cite{konevcny2016federated}\cite{caldas2018leaf}\cite{li2019federated}\cite{yang2019federated}\cite{bonawitz2019towards} is currently the dominant framework for distributed training of machine learning models under communication- and privacy constraints. Federated Learning assumes the clients to be congruent, i.e. that one central model can fit all client's distributions at the same time. Different authors have investigated the convergence properties of Federated Learning in congruent iid and non-iid scenarios: \cite{lin2018don},\cite{sattler2018sparse},\cite{sattler2019robust} and \cite{zhao2018federated} perform an empirical investigation, \cite{li2019convergence}, \cite{jiang2018linear}, \cite{yu2018parallel} and \cite{sahu2018convergence} prove convergence guarantees. As argued in section \ref{sec:intro} conventional Federated Learning is not able to deal with the challenges of incongruent data distributions.
Other distributed training frameworks \cite{koloskova2019decentralized1}\cite{koloskova2019decentralized2}\cite{stich2018sparsified}\cite{smith2016cocoa} are facing the same issues.  

The natural framework for dealing with incongruent data is Multi-Task Learning \cite{caruana1997multitask}\cite{jacob2009clustered}\cite{kumar2012learning}. An overview over recent techniques for multi-task learning in deep neural networks can be found in \cite{ruder2017overview}. However all of these techniques are applied in a centralized setting in which all data resides at one location and the server has full control over and knowledge about the optimization process. Smith et al. \cite{smith2017federated} present MOCHA, which extends the multi-task learning approach to the Federated Learning setting, by explicitly modeling client similarity via a correlation matrix. However their method relies on alternating bi-convex optimization and is thus only applicable to convex objective functions and limited in it's ability to scale to massive client populations. 
Corinzia et al. \cite{corinzia2019variational} model the connectivity structure between clients and server as a Bayesian network and perform variational inference during learning. Although their method can handle non-convex models, it is expensive to generalize to large federated networks as the client models are refined sequentially. 

Finally, Ghosh et al. \cite{ghosh2019robust} propose a clustering approach, similar to the one presented in this paper. However their method differs from ours in the following key aspects: Most significantly they use $l2$-distance instead of cosine similarity to determine the distribution similarity of the clients. This approach has the severe limitation that it only works if the client's risk functions are convex and the minima of different clusters are well separated. The $l2$-distance also is not able to distinguish congruent from incongruent settings. This means that the method will incorrectly split up clients in the conventional congruent non-iid setting described in \cite{mcmahan2016communication}. Furthermore, their approach is not adaptive in the sense that the decision whether to cluster or not has to be made after the first communication round. In contrast, our method can be applied to arbitrary Federated Learning problems with non-convex objective functions. We also note that we have provided theoretical considerations that allow a systematic understanding of the novel CFL framework.

%

\section{Implementation Considerations}
\label{sec:implementation}

\subsection{Weight-Updates as generalized Gradients}
\label{sec:weight-updates}
Theorem \ref{theo:separation_theorem} makes a statement about the cosine similarity between \emph{gradients} of the empirical risk function. In Federated Learning however, due to constraints on both the memory of the client devices and their communication budged, instead commonly weight-updates as defined in \eqref{eq:sgd} are computed and communicated. In order to deviate as little as possible from the classical Federated Learning algorithm it would hence be desirable to generalize result \ref{theo:separation_theorem} to weight-updates. It is commonly conjectured (see e.g. \cite{lin2017deep}) that accumulated mini-batch gradients approximate the full-batch gradient of the objective function. Indeed, for a sufficiently smooth loss function and low learning rate, a weight update computed over one epoch approximates the direction of the true gradient since by Taylor approximation we have
\begin{align}
&\nabla_\theta r(\theta_{\tau}+\eta\Delta\theta_{\tau-1}, D_\tau)\\&=\nabla_\theta r(\theta_{\tau}, D_\tau)+\eta\Delta\theta_{\tau-1}\nabla^2_\theta r(\theta_{\tau}, D_\tau)+\mathcal{O}(\|\eta\Delta\theta_{\tau-1}\|^2)\\
&=\nabla_\theta r(\theta_{\tau}, D_\tau)+R
\end{align}
where $R$ can be bounded in norm. Hence, by recursive application of the above result it follows
\begin{align}
\Delta\theta = \sum_{\tau=1}^T\nabla_\theta r(\theta_\tau, D_\tau) \approx \sum_{\tau=1}^T\nabla_\theta r(\theta_1, D_\tau) =\nabla_\theta r(\theta_1,D).
\end{align}  
In the remainder of this work we will compute cosine similarities between weight-updates instead of gradients according to
\begin{align}
\alpha_{i,j}:= \frac{\langle\Delta\theta_i , \Delta\theta_j \rangle}{\|\Delta\theta_i \|\|\Delta\theta_j \|}, i,j\in c
\end{align} 
Our experiments in section \ref{sec:experiments} will demonstrate that computing cosine similarities based on weight-updates in practice surprisingly achieves \emph{even better} separations than computing cosine similarities based on gradients.

\subsection{Preserving Privacy}
\label{sec:privacy}
Every machine learning model carries information about the data it has been trained on. For example the bias term in the last layer of a neural network will typically carry information about the label distribution of the training data. 
Different authors have demonstrated that information about a client's input data ("$x$") can be inferred from the weight-updates it sends to the server via model inversion attacks \cite{bhowmick2018protection}\cite{hitaj2017deep}\cite{fredrikson2015model}\cite{carlini2018secret}\cite{melis2018exploiting}. In privacy sensitive situations it might be necessary to prevent this type of information leakage from clients to server with mechanisms like the ones presented in \cite{bonawitz2017practical}.
Luckily, Clustered Federated Learning can be easily augmented with an encryption mechanism that achieves this end.
As both the cosine similarity between two clients' weight-updates and the norms of these updates are invariant to orthonormal transformations $P$ (such as permutation of the indices),
\begin{align}
\frac{\langle\Delta\theta_i,\Delta\theta_j\rangle}{\|\Delta\theta_i\|\|\Delta\theta_j\|} = \frac{\langle P\Delta\theta_i,P\Delta\theta_j\rangle}{\|P\Delta\theta_i\|\|P\Delta\theta_j\|}
\end{align}
a simple remedy is for all clients to apply such a transformation operator to their updates before communicating them to the server. After the server has averaged the updates from all clients and broadcasted the average back to the clients they simply apply the inverse operation
\begin{align}
\Delta\theta=\frac{1}{n}\sum_{i=1}^n\Delta\theta_i = P^{-1}(\frac{1}{n}\sum_{i=1}^n P\Delta\theta_i)
\end{align}
and the Federated Learning protocol can resume unchanged. Other multi-task learning approaches require direct access to the client's data and hence can not be used together with encryption, which represents a distinct advantage for CFL in privacy sensitive situations.

\subsection{Varying Client Populations and Parameter Trees}
\label{sec:tree}

\begin{figure}
\centering
\includegraphics[width=0.5\textwidth]{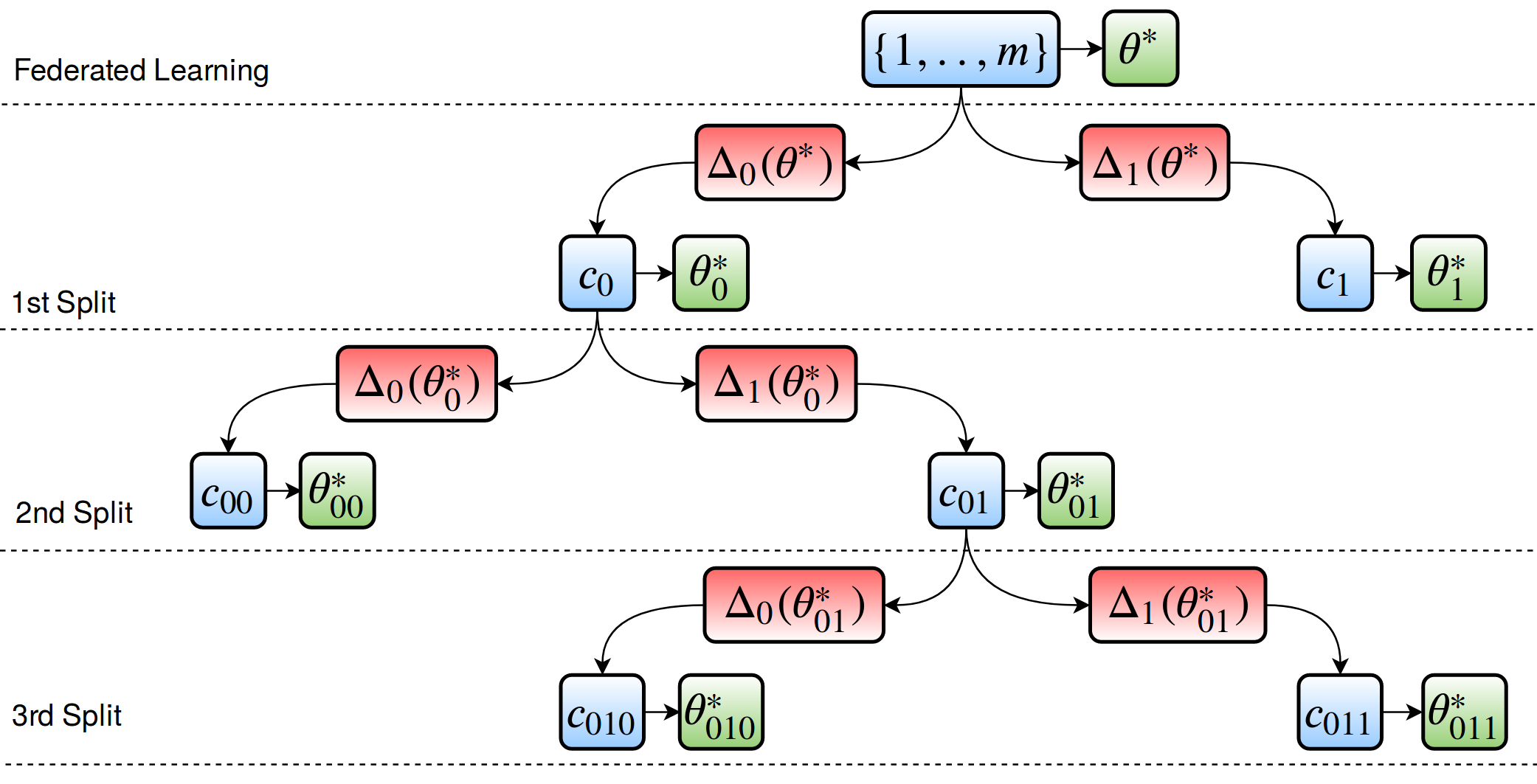}
\caption{An exemplary parameter tree created by Clustered Federated Learning. At the root node resides the conventional Federated Learning model, obtained by converging to a stationary point $\theta^*$ of the FL objective over all clients $\{1,..,m\}$. In the next layer, the client population has been split up into two groups, according to their cosine similarities and every subgroup has again converged to a stationary point $\theta^*_0$ respective $\theta^*_1$. Branching continues recursively until no stationary solution satisfies the splitting criteria. In order to quickly assign new clients to a leaf model, at each edge $e$ of the tree the server caches the pre-split weight-updates $\Delta_e$ of all clients belonging to the two different sub-branches. This way the new client can be moved down the tree along the path of highest similarity.}
\label{fig:tree}
\end{figure}
Up until now we always made the assumption that all clients participate from the beginning of training. Clustered Federated Learning however is flexible enough to handle client populations that vary over time.

In order to incorporate this functionality, the server, while running CFL, needs to build a parameter tree $T=(V,E)$ with the following properties:
\begin{itemize}
\item The tree contains a node $v\in V$ for every (intermediate) cluster $c_v$ computed by CFL
\item Both $c_v$ and the corresponding stationary solution $\theta^*_v$ obtained by running the Federated Learning Algorithm \ref{alg:FL} on cluster $c_v$ are cached at node $v$
\item At the root of the tree $v_{root}$ resides the Federated Learning solution over the entire client population with $c_{v_{root}}=\{1,..,m\}$.
\item If the cluster $c_{v_{child}}$ was created by bi-partitioning the cluster $c_{v_{parent}}$ in CFL then the nodes $v_{parent}$ and $v_{child}$ are connected via a directed edge $e\in E$
\item At every edge $e(v_{parent}\rightarrow v_{child})$ the pre-split weight-updates of the children clients
\begin{align}
\Delta_e=\{\text{SGD}_n(\theta^*_{v_{parent}}, D_i)-\theta_{v_{parent}^*}|i\in c_{v_{child}}\}
\end{align}
are cached

\end{itemize}
An exemplary parameter tree is shown in Figure \ref{fig:tree}.
When a new client joins the training it can get assigned to a leaf cluster by iteratively traversing the parameter tree from the root to a leaf, always moving to the branch which contains the more similar client updates according to Algorithm \ref{alg:tree}.

Another feature of building a parameter tree is that it allows the server to provide every client with \emph{multiple} models at varying specificity. On the path from root to leaf, the models get more specialized with the most general model being the FL model at the root. Depending on application and context, a CFL client could switch between models of different generality. Furthermore a parameter tree allows us to ensemble multiple models of different specificity together. We believe that investigations along those lines are a promising direction of future research.

\ \\
Putting all pieces from the previous sections together, we arrive at a protocol for general privacy-preserving CFL which is described in Algorithm \ref{alg:cryptoCFL}

\begin{algorithm}
\caption{Assigning new Clients to a Cluster}
\label{alg:tree}
\DontPrintSemicolon
\textbf{Input:} new client with data $D_{new}$, parameter tree $T=(V,E)$ \\
\textbullet~ $v\leftarrow v_{root}$\\
\While{$|\textnormal{Children}(v)|>0$}{
\textbullet~ $v_0, v_1 \leftarrow \text{Children}(v)$\\
\textbullet~ $\Delta\theta_{new}\leftarrow\text{SGD}_n(\theta^*_v, D_{new})-\theta^*_v$\\
\textbullet~ $\alpha_0\leftarrow \max_{\Delta\theta\in \Delta_{(v\rightarrow v_1)}}\alpha(\Delta\theta_{new}, \Delta\theta)$\\
\textbullet~ $\alpha_1\leftarrow \max_{\Delta\theta\in \Delta_{(v\rightarrow v_2)}}\alpha(\Delta\theta_{new}, \Delta\theta)$\\
\uIf{$\alpha_0>\alpha_1$}{
\textbullet~ $v\leftarrow v_0$\\
}
\Else{
\textbullet~ $v\leftarrow v_1$\\
}
}
\Return $c_v$, $\theta^*_v$
\end{algorithm}
\begin{algorithm}
\caption{Clustered Federated Learning with Privacy Preservation and Weight-Updates}
\label{alg:cryptoCFL}
\DontPrintSemicolon
\textbf{input:} initial parameters $\theta_0$, branching parameters $\varepsilon_1, \varepsilon_2 > 0$, empirical risk approximation error bound $\gamma_{max}\in[0,1)$, number of local iterations/ epochs $n$\\
\textbf{outout:} improved parameters on every client $\theta_i$\\
\textbf{init:} set initial clusters $\mathcal{C}=\{ \{ 1,..,m \} \}$, set initial models $\theta_i\leftarrow\theta_0$  $\forall i=1,..,m$, set initial update $\Delta\theta_c\leftarrow 0$ $\forall c\in \mathcal{C}$, clients exchange random seed to create permutation operator $P$ (optional, otherwise set $P$ to be the identity mapping)\\
\While{not converged}{
\For{$i=1,..,m$ \textbf{in parallel}}{
\underline{Client $i$ does:}\\
\textbullet~ $\theta_i \leftarrow \theta_i+P^{-1}\Delta\theta_{c(i)}$\\
\textbullet~ $\Delta \theta_i \leftarrow P(\text{SGD}_n(\theta_i, D_i)-\theta_i)$ \\
}

\underline{Server does:}\\
\textbullet~ $\mathcal{C}_{tmp} \leftarrow \mathcal{C}$\\
\For{$c \in \mathcal{C}$}{
\textbullet~ $\Delta\theta_c \leftarrow \frac{1}{|c|}\sum_{i\in c}\Delta \theta_i$ \\
\If{$\|\Delta\theta_c\|<\varepsilon_1$  \textbf{and} $\max_{i\in c}\|\Delta \theta_i\|>\varepsilon_2$}{
\textbullet~ $\alpha_{i,j}\leftarrow\frac{\langle \Delta \theta_i,\Delta \theta_j\rangle}{\|\Delta \theta_i\|\|\Delta \theta_j\|}$\\

\textbullet~ $c_1,c_2\leftarrow \arg\min_{c_1\cup c_2 = c}(\max_{i\in c_1, j\in c_2} \alpha_{i,j})$\\
\textbullet~ $\alpha_{cross}^{max} \leftarrow \max_{i\in c_1, j\in c_2}\alpha_{i,j}$

\If{$\gamma_{max} < \sqrt{\frac{1-\alpha_{cross}^{max}}{2}}$}{
\textbullet~ $\mathcal{C}_{tmp}\leftarrow(\mathcal{C}_{tmp}\setminus c) \cup c_1\cup c_2$
}
}
}
\textbullet~ $\mathcal{C} \leftarrow \mathcal{C}_{tmp}$\\
}
\Return $\theta$

\end{algorithm}

%
%

\section{Experiments}
\label{sec:experiments}

\subsection{Practical Considerations}
\label{sec:practical}
In section \ref{sec:theory} we showed that the cosine similarity criterion does distinguish different incongruent clients under three conditions: (a) Federated Learning has converged to a stationary point $\theta^*$, (b) Every client holds enough data s.t. the empirical risk approximates the true risk, (c) cosine similarity is computed between the full gradients of the empirical risk. In this section we will demonstrate that in practical problems none of these conditions have to be fully satisfied. Instead, we will find that CFL is able to correctly infer the clustering structure even if clients only hold small datasets and are trained to an approximately stationary solution of the Federated Learning objective. Furthermore we will see that cosine similarity can be computed between weight-updates instead of full gradients, which even improves performance. 

In the experiments of this section we consider the following Federated Learning setup: All experiments are performed on either the MNIST \cite{lecun1998mnist} or CIFAR-10 \cite{krizhevsky2014cifar} dataset using $m=20$ clients, each of which belonging to one of $k=4$ clusters. Every client is assigned an equally sized random subset of the total training data. To simulate an incongruent clustering structure, every clients' data is then modified by randomly swapping out two labels, depending on which cluster a client belongs to. For example, in all clients belonging to the first cluster, data points labeled as "1" could be relabeled as "7" and vice versa, in all clients belonging to the second cluster "3" and "5" could be switched out in the same way, and so on. This relabeling ensures that both $\varphi(x)$ and $\varphi(y)$ are approximately the same across all clients, but the conditionals $\varphi(y|x)$ diverge between different clusters. We will refer to this as "label-swap augmentation" in the following. In all experiments we train multi-layer convolutional neural networks and adopt a standard Federated Learning strategy with 3 local epochs of training. 
We report the separation gap (Definition \ref{def:separationgap})
\begin{align}
g(\alpha):=\alpha_{intra}^{min}-\alpha_{cross}^{max}
\end{align}
which according to Corollary \ref{corr:1} tells us whether CFL will correctly bi-partition the clients:
\begin{align}
g(\alpha)>0 \Leftrightarrow \text{"Correct Clustering"}
\end{align}

\textbf{Number of Data points:}
We start out by investigating the effects of data set size on the cosine similarity. We randomly subsample from each client's training data to vary the number of data points on every client between 10 and 200 for MNIST and 100 and 2400 for CIFAR. For every different local data set size we run Federated Learning for 50 communication rounds, after which training progress has come mostly to halt and we can expect to be close to a stationary point. After round 50, we compute the pairwise cosine similarities between the weight-updates and the separation gap $g(\alpha)$. The results are shown in Figure \ref{fig:ndata}. As expected, $g(\alpha)$ grows monotonically with increasing data set size. On the MNIST problem as little as 20 data points on every client are sufficient to achieve correct bi-partitioning in the sense of Definition \ref{def:1}. On the more difficult CIFAR problem a higher number of around 500 data points is necessary to achieve correct bi-partitioning.  
\begin{figure}
\centering
\includegraphics[width=0.5\textwidth]{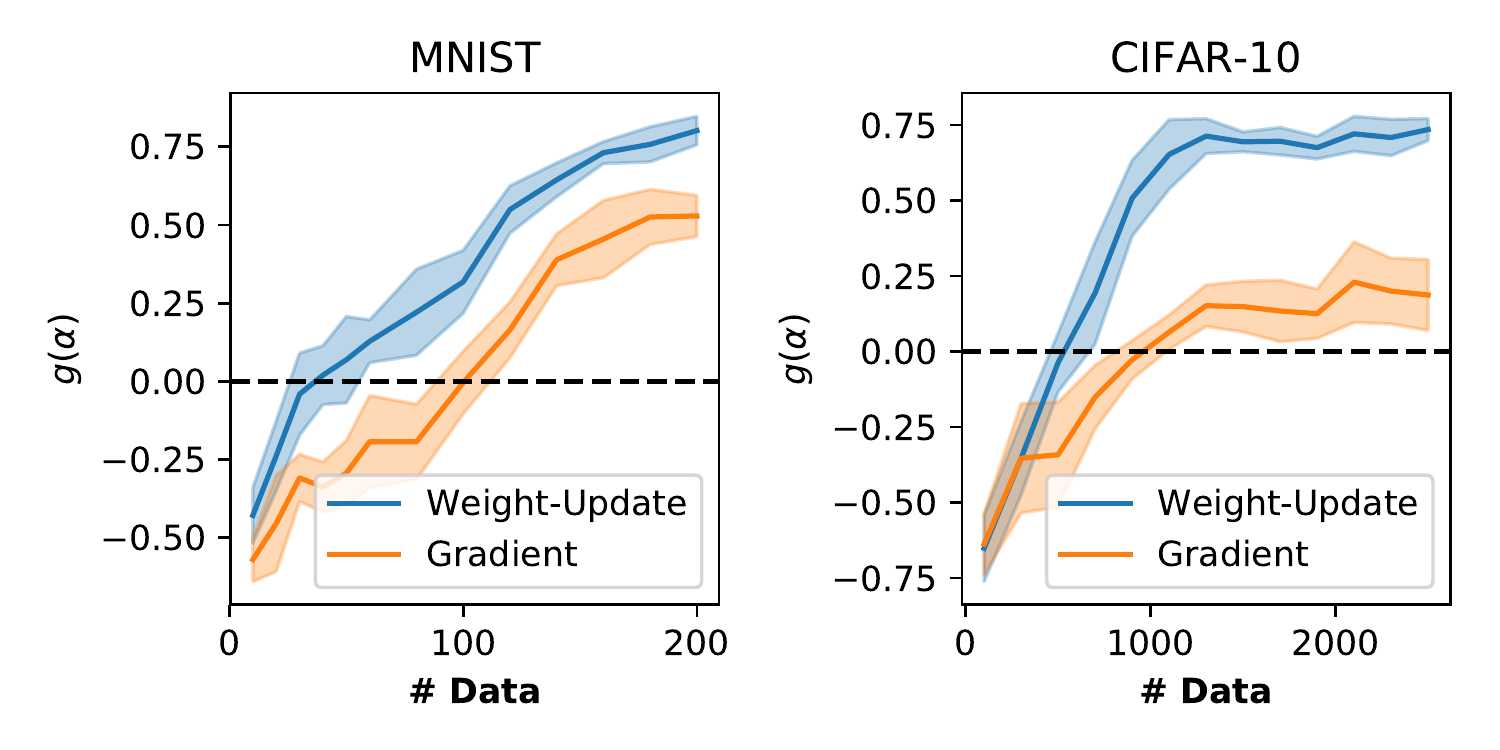}
\caption{Separation gap $g(\alpha)$ as a function of the number of data points on every client for the label-swap problem on MNIST and CIFAR. From Corollary \ref{corr:1} we know that CFL will always find a correct bi-partitioning if $g(\alpha)>0$. On MNIST this is already satisfied if clients hold as little as 20 data points if weight-updates are used for the computation of the similarity $\alpha$.}
\label{fig:ndata}
\end{figure}

\textbf{Proximity to Stationary Solution:}
\label{sec:nrounds}
Next, we investigate the importance of proximity to a stationary point $\theta^*$ for the clustering. Under the same setting as in the previous experiment we reduce the number of data points on every client to 100 for MNIST and to 1500 for CIFAR and compute the pairwise cosine similarities and the separation gap after each of the first 50 communication rounds. The results are shown in Figure \ref{fig:nrounds}. Again, we see that the separation quality monotonically increases with the number of communication rounds. On MNIST and CIFAR as little as 10 communication rounds are necessary to obtain a correct clustering. 
\begin{figure}
\centering
\includegraphics[width=0.5\textwidth]{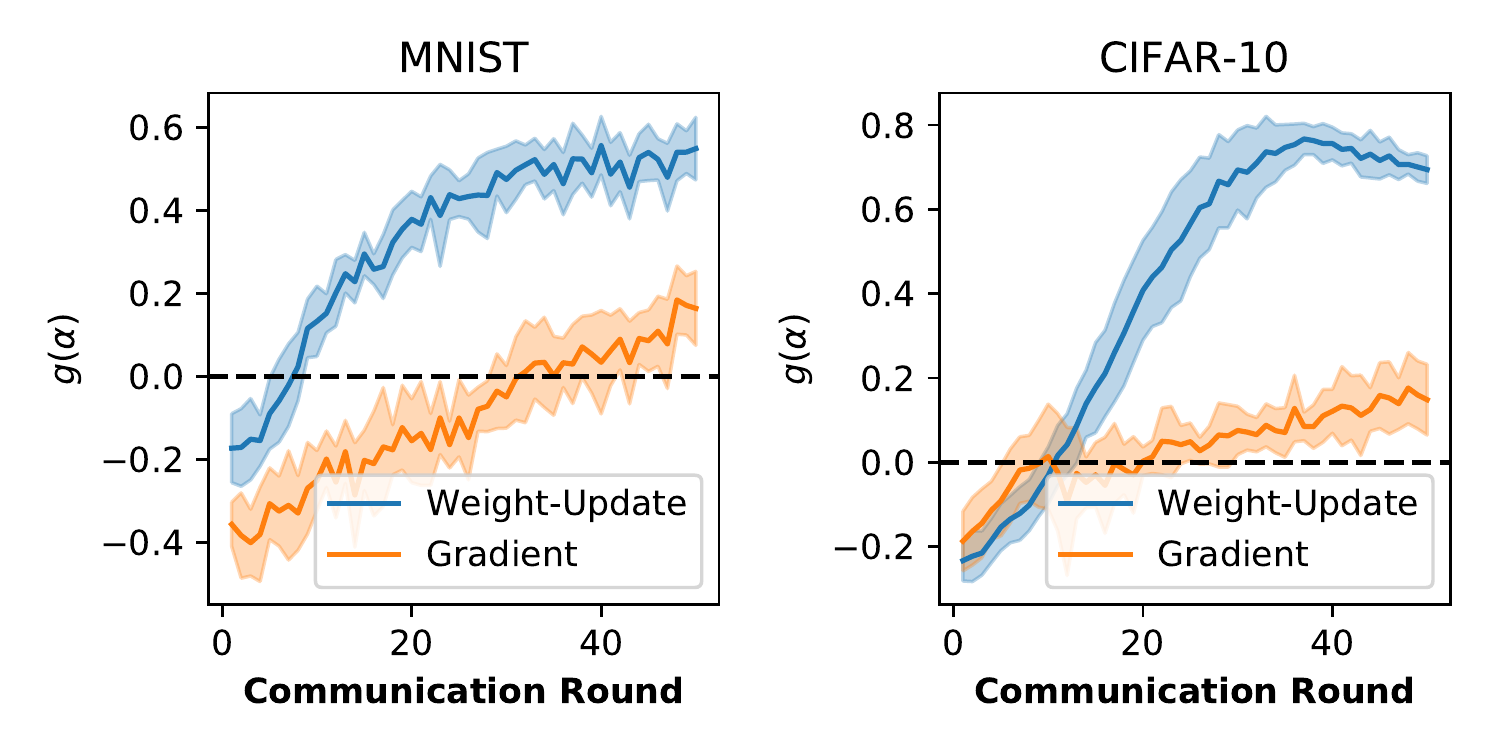}
\caption{Separation gap $g(\alpha)$ as a function of the number of communication rounds for the label-swap problem on MNIST and CIFAR. The separation quality monotonically increases with the number of communication rounds of Federated Learning. Correct separation in both cases is already achieved after around 10 communication rounds if $\alpha$ is computed using weight-updates.}
\label{fig:nrounds}
\end{figure}

\textbf{Weight-Updates instead of Gradients:}
In both the above experiments we computed the cosine similarities $\alpha$ based on either the full gradients
\begin{align}
\alpha_{i,j}=\frac{\langle \nabla_\theta r_i(\theta),\nabla_\theta r_j(\theta) \rangle}{\|\nabla_\theta r_i(\theta)\|\|\nabla_\theta r_j(\theta)\|} ~~~~~\text{   ("Gradient") } 
\end{align}
or Federated weight-updates 
\begin{align}
\alpha_{i,j}=\frac{\langle \Delta\theta_i,\Delta\theta_j \rangle}{\|\Delta\theta_i\|\|\Delta\theta_j\|}~~~~~\text{   ("Weight-Update") } 
\end{align}
over 3 epochs. Interestingly, weight-updates seem to provide even better separation $g(\alpha)$ with fewer data points and at a greater distance to a stationary solution.
This comes in very handy as it allows us to leave the Federated Learning communication protocol unchanged. In all following experiments we will compute cosine similarities based on weight-updates instead of gradients.

\subsection{Distinguishing Congruent and Incongruent Clients}
\label{sec:criteria}

\begin{figure}
\centering
\includegraphics[width=0.4\textwidth]{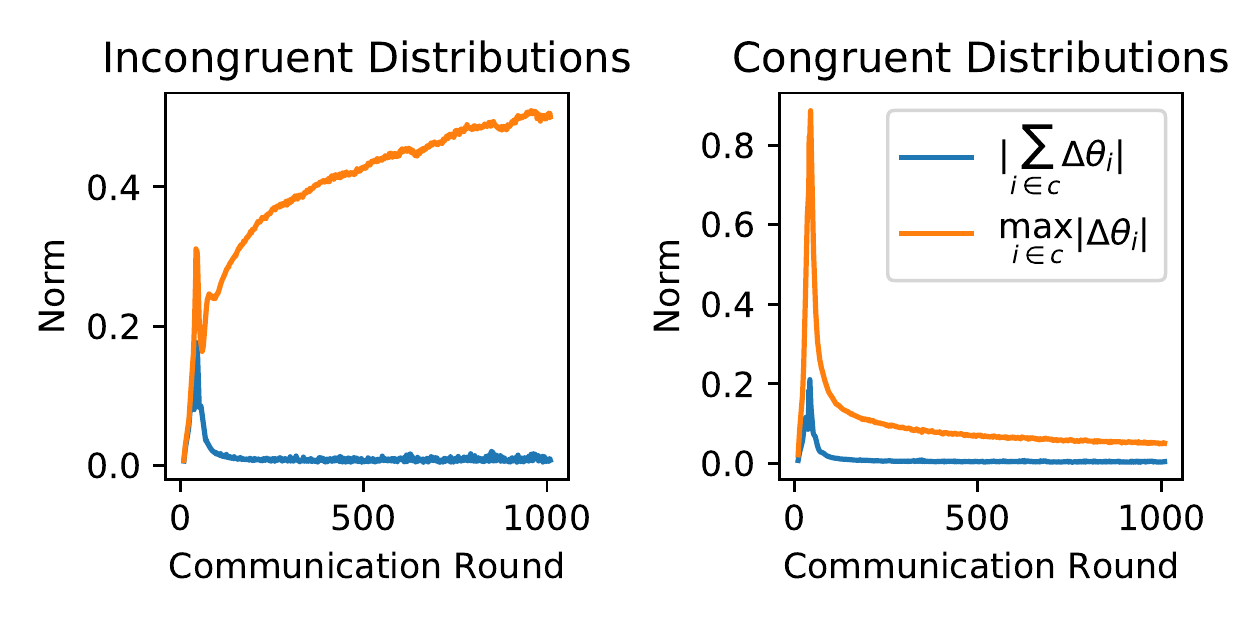}
\caption{Experimental verification of the norm criteria \eqref{eq:clientnorm} and \eqref{eq:servernorm}. Displayed is the development of gradient norms over the course of 1000 communication rounds of Federated Learning with two clients holding data from incongruent (left) and congruent distributions (right). In both cases Federated Learning converges to a stationary point of $F(\theta)$ and the average update norm \eqref{eq:servernorm} goes to zero. In the congruent case the maximum norm of the client updates \eqref{eq:clientnorm} decreases along with the server update norm, while in contrast in the incongruent case it stagnates and even increases.}
\label{fig:norms}
\end{figure}
In this subsection, we experimentally verify the validity of the clustering criteria \eqref{eq:servernorm} and \eqref{eq:clientnorm} in a Federated Learning experiment on MNIST with two clients holding data from incongruent and congruent distributions. In the congruent case client one holds all training digits "0" to "4" and client two holds all training digits "5" to "9". In the incongruent case, both clients hold a random subset of the training data, but the distributions are modified according to the "label swap" rule described above. Figure \ref{fig:norms} shows the development of the average update norm (equation \eqref{eq:servernorm}) and the maximum client norm (equation \eqref{eq:clientnorm}) over the course of 1000 communication rounds. As predicted by the theory, in the congruent case the average client norm converges to zero, while in the incongruent case it stagnates and even increases over time. In both cases the server norm tends to zero, indicating convergence to a stationary point (see Figure \ref{fig:norms}). 


%
%
%
%
\subsection{Clustered Federated Learning}
In this section, we apply CFL as described in Algorithm \ref{alg:cryptoCFL} to different Federated Learning setups, which are inspired by our motivating examples in the introduction. In all experiments, the clients perform 3 epochs of local training at a batch-size of 100 in every communication round. 


\begin{figure}
\centering
\includegraphics[width=0.5\textwidth]{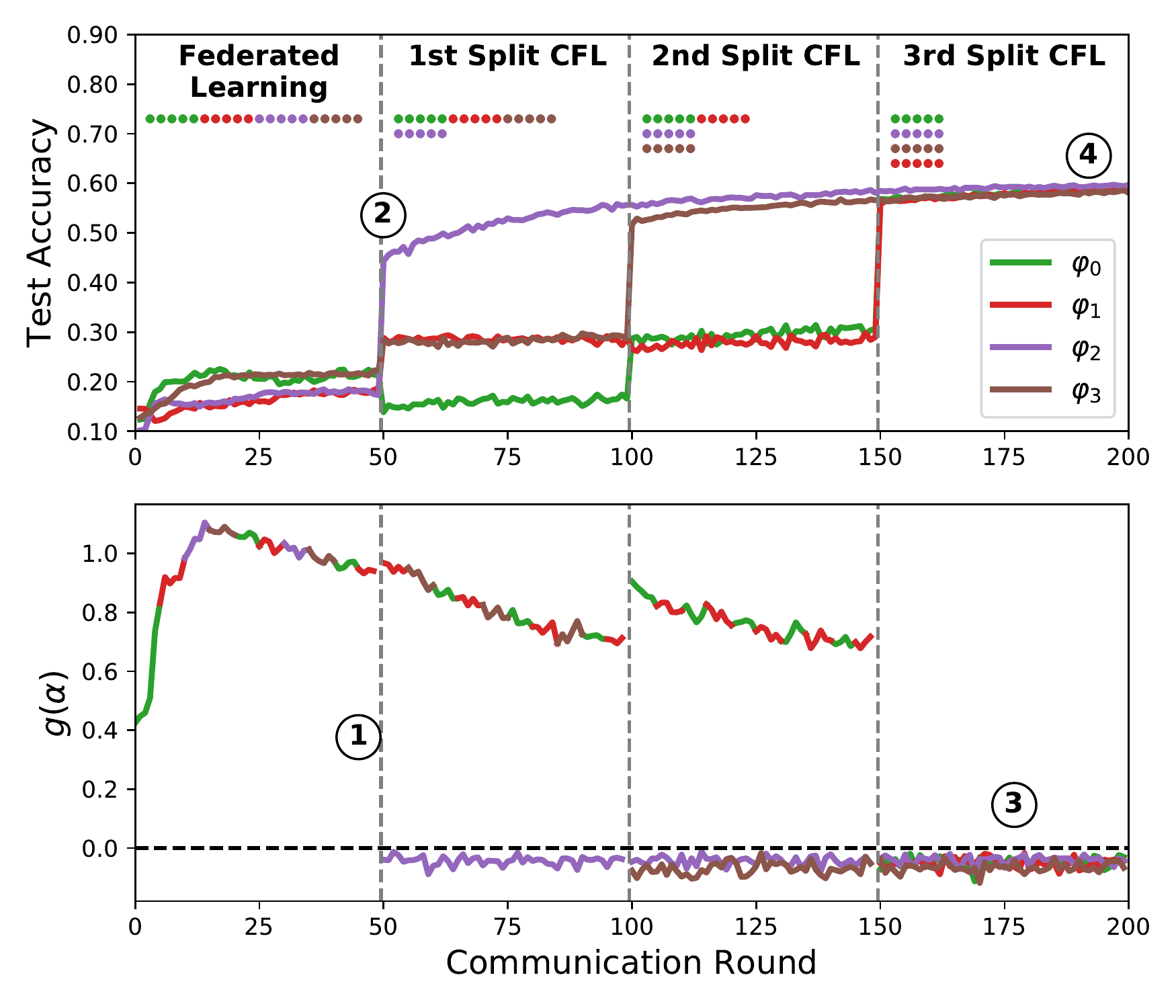}
\caption{CFL applied to the "permuted labels problem" on CIFAR with 20 clients and 4 different permutations. The top plot shows the accuracy of the trained model(s) on their corresponding validation sets. The bottom plot shows the separation gaps $g(\alpha)$ for all different clusters. After an initial 50 communication rounds a large separation gap has developed and a first split separates out the purple group of clients, which leads to an immediate drastic increase of accuracy for these clients. In communication rounds 100 and 150 this step is repeated until all clients with incongruent distributions have been separated. After the third split, the model accuracy for all clients has more than doubled and the separation gaps in all clusters have dropped to below zero which indicates that the clustering is finalized.}
\label{fig:cifar_shifty}
\end{figure}
\begin{figure}
\centering
\includegraphics[width=0.5\textwidth]{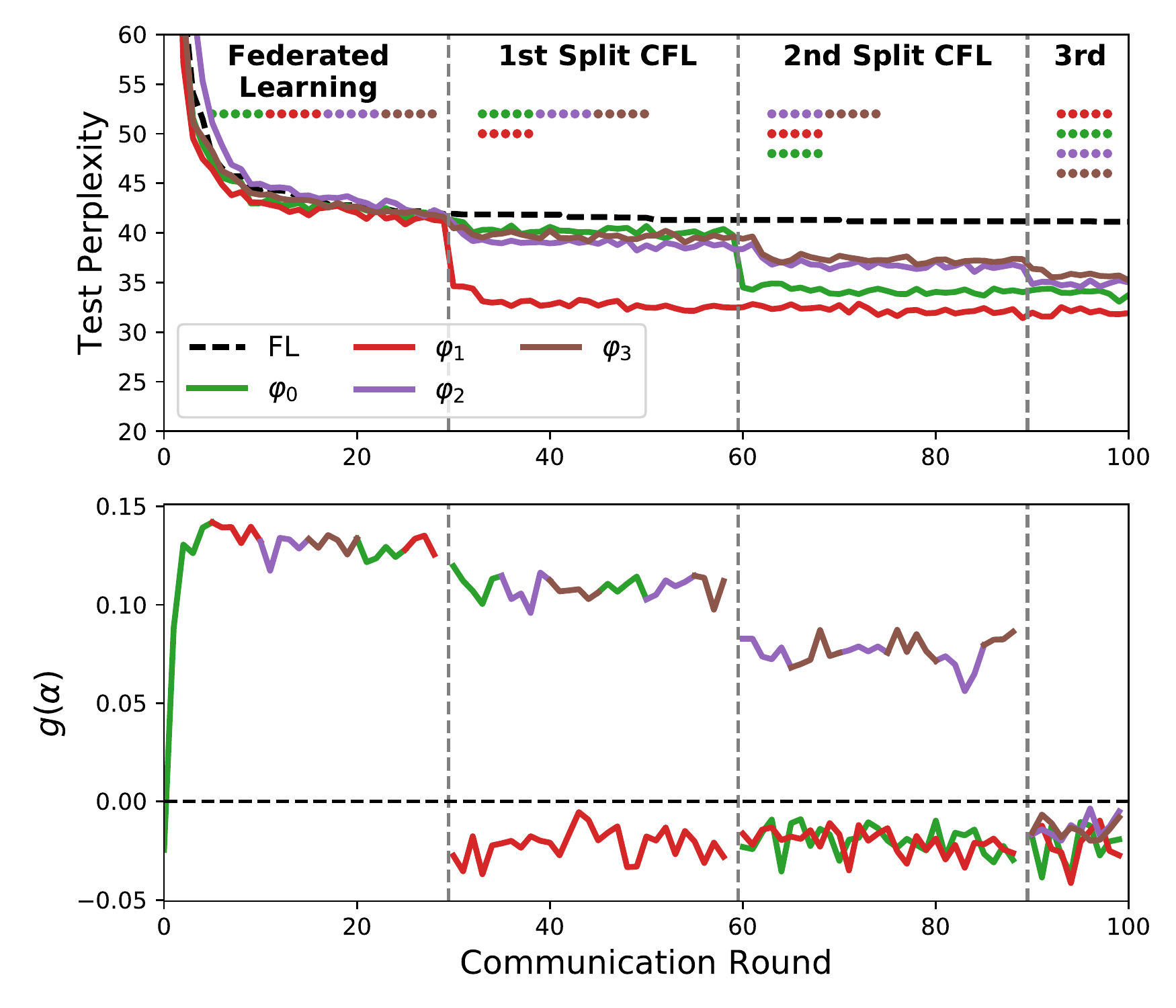}
\caption{CFL applied to the Ag-News problem. The top plot shows the perplexity achieved by the different clients on their local test set (lower is better). The clients are separated in communication rounds 30, 60 and 90. After the final separation the perplexity of all clients on their local test set has dropped to less than 36, while the Federated Learning solution (black, dotted) still stagnates at a perplexity of 42. }
\label{fig:agnews}
\end{figure}

\textbf{Label permutation on Cifar-10}:  We split the CIFAR-10 training data randomly and evenly among $m=20$ clients, which we group into $k=4$ different clusters. All clients belonging to the same cluster apply the same random permutation $P_{c(i)}$ to their labels such that their modified training and test data is given by 
\begin{align}
\hat{D_i}=\{(x,P_{c(i)}(y))|(x,y)\in D_i\}
\end{align}
respective
\begin{align}
\hat{D^{test}_i} =\{(x,P_{c(i)}(y))|(x,y)\in D^{test}\}.
\end{align}
The clients then jointly train a 5-layer convolutional neural network on the modified data using CFL with 3 epochs of local training at a batch-size of 100. Figure \ref{fig:cifar_shifty} (top) shows the joint training progression: In the first 50 communication rounds, all clients train one single model together, following the conventional Federated Learning protocol. After these initial 50 rounds, training has converged to a stationary point of the Federated Learning objective and the client test accuracies stagnate at around 20\%. Conventional Federated Learning would be finalized at this point. At the same time, we observe (Figure \ref{fig:cifar_shifty}, bottom) that a distinct gap $g(\alpha)=\alpha_{intra}^{min}-\alpha_{cross}^{max}$ has developed (\circled{1}), indicating an underlying clustering structure. In communication round 50 the client population is therefore split up for the first time, which leads to an immediate 25\% increase in validation accuracy for all clients belonging to the "purple" cluster which was separated out \circled{2}. Splitting is repeated in communication rounds 100 and 150 until all clusters have been separated and $g(\alpha)$ has dropped to below zero in all clusters (\circled{3}), which indicates that clustering is finalized. At this point the accuracy of all clients has more than doubled the one achieved by the Federated Learning solution and is now at close to 60\% \circled{4}. This underlines, that after standard FL, our novel CFL can detect, the necessity for subsequent splitting and clustering which enables arriving at significantly higher performance. In addition, the cluster structure found can potentially be illuminating as it provides interesting insight about the composition of the complex underlying data distribution.

\textbf{Language Modeling on Ag-News}: The Ag-News corpus is a collection of 120000 news articles belonging to one of the four topics 'World', 'Sports', 'Business' and 'Sci/Tech'. We split the corpus into 20 different sub-corpora of the same size, with every sub-corpus containing only articles from one topic and assign every corpus to one client. Consequently the clients form four clusters based on what type of articles they hold. Every Client trains a two-layer LSTM network to predict the next word on it's local corpus of articles. Figure \ref{fig:agnews} shows 100 communication rounds of multi-stage CFL applied to this distributed learning problem. As we can see, Federated Learning again converges to a stationary solution after around 30 communication rounds. At this solution all clients achieve a perplexity of around 43 on their local test set. After the client population has been split up in communication rounds 30, 60 and 90, the four true underlying clusters are discovered. After the 100th communication round the perplexity of all clients has dropped to less than 36. The Federated Learning solution, trained over the same amount of communication rounds, still stagnates at an average perplexity of 42.

\section{Conclusion}
In this paper we presented Clustered Federated Learning, a framework for Federated Multi-Task Learning that can improve any existing Federated Learning Framework by enabling the participating clients to learn more specialized models. 
Clustered Federated Learning makes use of our theoretical finding, that (at any stationary solution of the Federated Learning objective) the cosine similarity between the weight-updates of different clients is highly indicative of the similarity of their data distributions.
This crucial insight allows us to provide strong mathematic guarantees on the clustering quality under mild assumptions on the clients and their data, even for arbitrary non-convex objectives.

We demonstrated that CFL can be implemented in a privacy preserving way and without having to modify the FL communication protocol.
Moreover, CFL is able to distinguish situations in which a single model can be learned from the clients' data from those in which this is not possible and only separates clients in the latter case. 

Our experiments on convolutional and recurrent deep neural networks show that CFL can achieve drastic improvements over the Federated Learning baseline in terms of classification accuracy / perplexity in situations where the clients' data exhibits a clustering structure.

In future work we plan to further investigate the utility of weight-updates for distribution similarity estimation. Another interesting direction of further research is whether our results will also generalize to compressed weight-updates \cite{caldas2018expanding}\cite{konevcny2016federated2}\cite{wiedemann2019deepcabac} or compact parameter representations \cite{wiedemann2019compact}.

\bibliographystyle{IEEEtran}
\bibliography{sample.bib}

\section{Supplement}
\subsection{Proving the Separation Theorem}

The Separation Theorem \ref{theo:separation_theorem} makes a statement about the cosine similarities between the gradients of the empirical risk functions $\nabla_\theta r_i(\theta^*)$ and $\nabla_\theta r_j(\theta^*)$, which are noisy approximations of the true risk gradients $\nabla_\theta R_{I(i)}(\theta^*)$, respective $\nabla_\theta R_{I(j)}(\theta^*)$. To simplify the notation let us first re-define 
\begin{align}
v_l=\nabla_\theta R_l(\theta^*), l=1,..,k
\end{align}
and
\begin{align}
X_i=\nabla_\theta r_i(\theta^*)-\nabla_\theta R_{I(i)}(\theta^*), i=1,..,m
\end{align}
Figure \ref{fig:proof_overview} shows a possible configuration in $d=2$ with $k=3$ different data generating distributions and their corresponding gradients $v_1$, $v_2$ and $v_3$. The empirical risk gradients $X_i+v_{i(i)}$, $i=1,..,m$ are shown as dashed lines. The maximum angles between gradients from the same data generating distribution are shown green, blue and purple in the plot. Among these, the green angle is the largest one $\sphericalangle_{intra}^{max}$. The plot also shows the optimal bi-partitioning into clusters 1 and 2 and the minimum angle between the gradient updates from any two clients in different clusters $\sphericalangle_{cross}^{min}$ is displayed in red. As long as 
\begin{align}
\sphericalangle_{intra}^{max}<\sphericalangle_{cross}^{min}
\end{align}
or equivalently
\begin{align}
\alpha_{intra}^{min}=\cos(\sphericalangle_{intra}^{max})>\cos(\sphericalangle_{cross}^{min})=\alpha_{cross}^{max}
\end{align}
  the clustering will always be correct.

\begin{figure}
\centering
\includegraphics[width=0.5\textwidth]{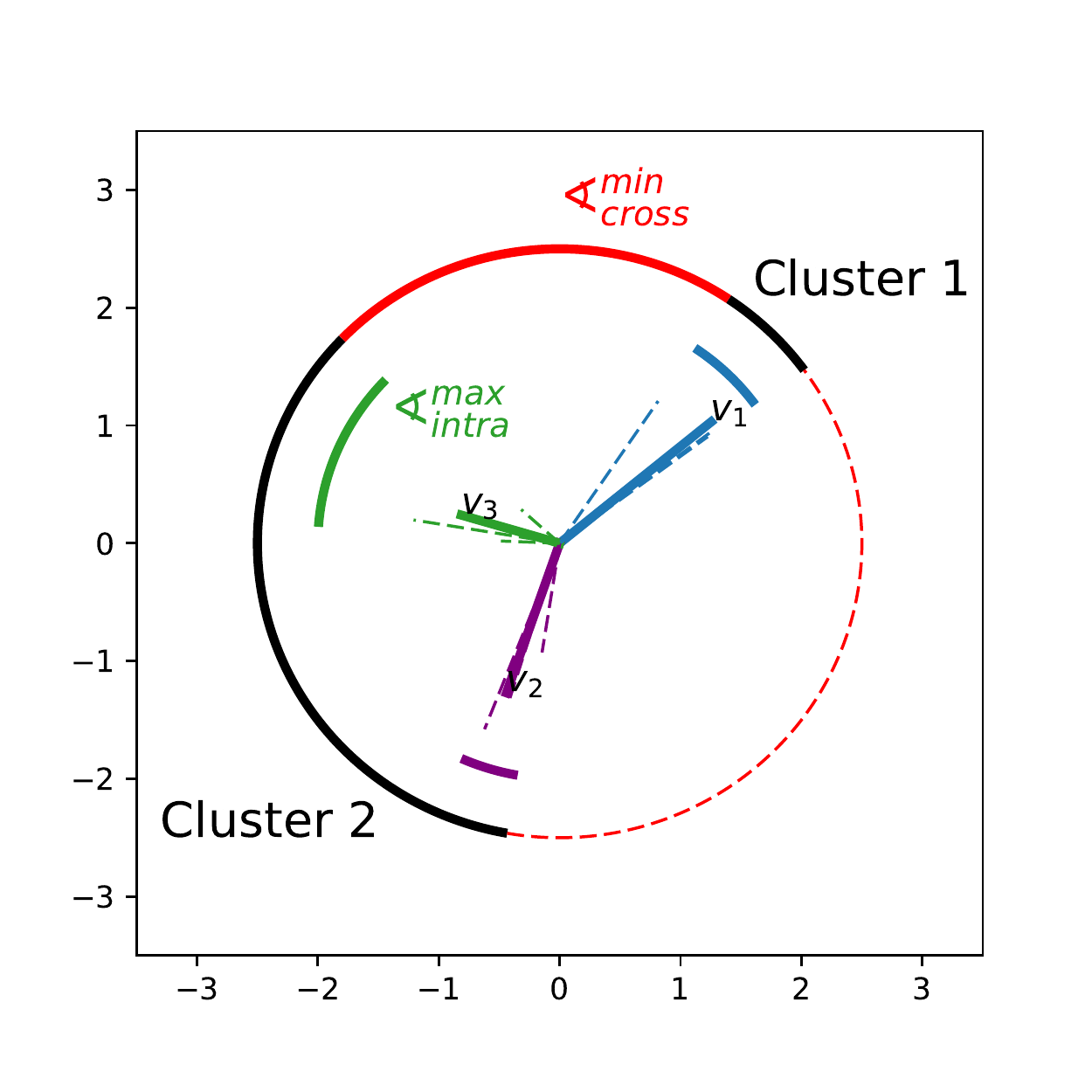}
\caption{Possible configuration in $d=2$ with $k=3$ different data generating distributions and their corresponding gradients $v_1$, $v_2$ and $v_3$. The empirical risk gradients $X_i+v_{i(i)}$, $i=1,..,m$ are shown as dashed lines. The maximum angles between gradients from the same data generating distribution are shown green, blue and purple in the plot. Among these, the green angle is the largest one $\sphericalangle_{intra}^{max}$. The vectors are optimally bi-partitioned into clusters 1 and 2 and the minimum angle between the gradient updates from any two clients in different clusters $\sphericalangle_{cross}^{min}$ is displayed in red.}
\label{fig:proof_overview}
\end{figure}

The proof of Theorem \ref{theo:separation_theorem} can be organized into three separate steps:
\begin{itemize}
\item In Lemma \ref{lemma:same}, we bound the cosine similarity between two noisy approximations of the same vector $\alpha_{intra}^{min}$ from below
\item In Lemma \ref{lemma:different}, we bound the cosine similarity between two noisy approximations of two different vectors from above
\item In Lemma \ref{lemma:kcluster}, we show that every set of vectors that sums to zero can be separated into two groups such that the cosine similarity between any two vectors from separate groups can be bounded from above
\item Lemma \ref{lemma:different} and \ref{lemma:kcluster} together will allow us to bound the cross cluster similarity $\alpha_{cross}^{max}$ from above
\end{itemize} 
\begin{lemma}
\label{lemma:same}
Let $v,X, Y\in\mathbb{R}^d$ with $\|X\|<\|v\|$ and $\|Y\|<\|v\|$ then
\begin{align}
\alpha(v+X,v+Y)\geq -\frac{\|X\|\|Y\|}{\|v\|^2}+\sqrt{1-\frac{\|X\|^2}{\|v\|^2}}\sqrt{1-\frac{\|Y\|^2}{\|v\|^2}}.
\end{align}
\end{lemma}

\begin{proof}
We are interested in vectors $X$ and $Y$ which maximize the angle between $v+X$ and $v+Y$. Since \begin{align}
\alpha(v+X,v+Y) = \cos(\sphericalangle(v+X,v+Y))
\end{align}
and $\cos$ is monotonically decreasing on $[0,\pi]$ such $X$ and $Y$ will minimize the cosine similarity $\alpha$. As $\|X\|<\|v\|$ and $\|Y\|<\|v\|$ the angle will be maximized if and only if $v$, $X$ and $Y$ share a common 2-dimensional hyperplane, $X$ is perpendicular to $v+X$ and $Y$ is perpendicular to $v+Y$ and $X$ and $Y$ point into opposite directions (Figure \ref{fig:proof_maxangle}). It then holds by the trigonometric property of the sine that
\begin{align}
\sin(\sphericalangle(v,v+X)) = \frac{\|X\|}{\|v\|}
\end{align}  
and
\begin{align}
\sin(\sphericalangle(v,v+Y)) = \frac{\|Y\|}{\|v\|}
\end{align} 
and hence 
\begin{align}
\cos(\sphericalangle(v+X,v+Y))&= \cos(\sphericalangle(v+X)+\sphericalangle(v+Y))\\
&\geq \cos(\sin^{-1}(\frac{\|X\|}{\|v\|})+\sin^{-1}(\frac{\|Y\|}{\|v\|})).
\end{align}
Since 
\begin{align}
\cos(\sin^{-1}(x)+\sin^{-1}(y))=-xy+\sqrt{1-x^2}\sqrt{1-y^2}
\end{align}
the result follows after re-arranging terms.

\end{proof}


\begin{figure}
\centering
\includegraphics[width=0.5\textwidth]{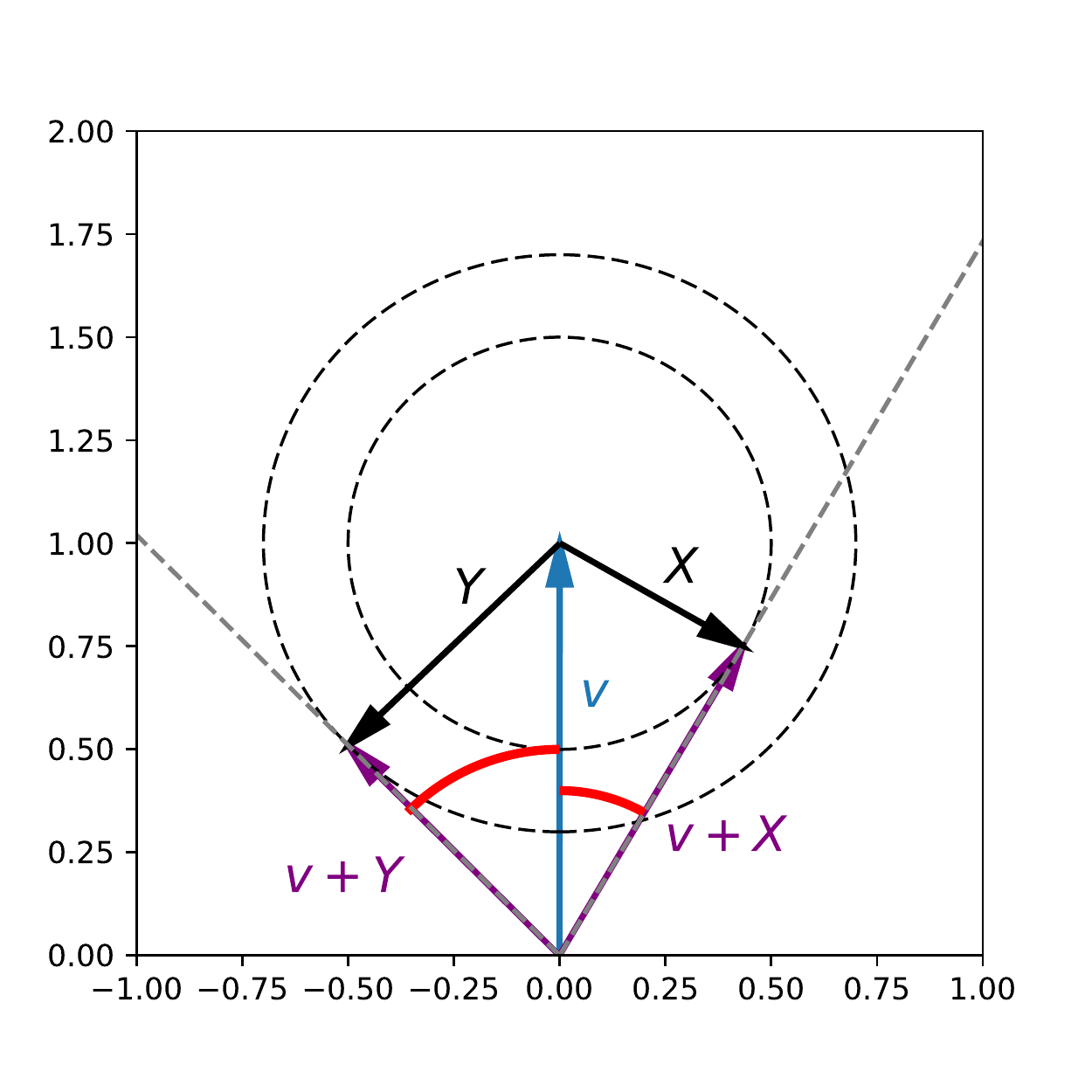}
\caption{We are interested in a configuration for which the angle between $v+X$ and $v+Y$ is maximized (red in the plot). As $\|X\|<\|v\|$ and $\|Y\|<\|v\|$ this is exactly the case if the line $\{\beta(v+X)|\beta\in \mathbb{R}\}$ is tangential to the circle with center $v$ and radius $\|X\|$ and the line $\{\beta(v+Y)|\beta\in \mathbb{R}\}$ is tangential to the circle with center $v$ and radius $\|Y\|$.}
\label{fig:proof_maxangle}
\end{figure}

\begin{lemma}
\label{lemma:different}
Let $v,w, X, Y\in\mathbb{R}^d$ with $\|X\|<\|v\|$, $\|Y\|<\|w\|$ and define
\begin{align}
h(v,w,X,Y):=-\frac{\|X\|\|Y\|}{\|v\|^2}+\sqrt{1-\frac{\|X\|^2}{\|v\|^2}}\sqrt{1-\frac{\|Y\|^2}{\|v\|^2}}
\end{align}
If
\begin{align}
\label{eq:cond}
\frac{\langle v, w\rangle}{\|v\|\|w\|}\leq h(v,w,X,Y)
\end{align}
then it holds
\begin{align}
\alpha(v+X,w+Y)\leq & \alpha(v,w)h(v,w,X,Y)\\
&+\sqrt{1-\alpha(v,w)^2}\sqrt{1-h(v,w,X,Y)^2}
\end{align}
\end{lemma}

\begin{proof}
Analogously to the argument in Figure \ref{fig:proof_maxangle}, the angle between $v+X$ and $w+Y$ is minimized, when $v$, $w$, $X$ and $Y$ share a common 2-dimensional hyperplane, $X$ is orthogonal to $v+X$, $Y$ is orthogonal to $w+Y$, and $X$ and $Y$ point towards each other. The minimum possible angle is then given by
\begin{align}
\sphericalangle(v+X,w+Y) &= \sphericalangle(v,w)-\sphericalangle(v,v+X)-\sphericalangle(w,w+Y) \\
&\geq \max(0, \\&\cos^{-1}(\frac{\langle v, w\rangle}{\|v\|\|w\|})\\-&\sin^{-1}(\frac{\|X\|}{\|v\|})+\\-&\sin^{-1}(\frac{\|Y\|}{\|v\|}))
\end{align}  
which can be simplified to
\begin{align}
\sphericalangle(v+X,w+Y) \geq \max(0, &\cos^{-1}(\frac{\langle v, w\rangle}{\|v\|\|w\|})\\-\cos^{-1}(-\frac{\|X\|\|Y\|}{\|v\|^2}&+\sqrt{1-\frac{\|X\|^2}{\|v\|^2}}\sqrt{1-\frac{\|Y\|^2}{\|v\|^2}}))
\end{align}
Under condition \eqref{eq:cond} then second term in the maximum is greater than zero and we get
\begin{align}
\cos(\sphericalangle&(v+X,v+Y)) \\\leq \cos(&\cos^{-1}(\frac{\langle v, w\rangle}{\|v\|\|w\|})\\-&\cos^{-1}(-\frac{\|X\|\|Y\|}{\|v\|^2}+\sqrt{1-\frac{\|X\|^2}{\|v\|^2}}\sqrt{1-\frac{\|Y\|^2}{\|v\|^2}}))\\
\leq \cos(&\cos^{-1}(\alpha(v,w))-\cos^{-1}(h(v,w,X,Y)))
\end{align}
Since 
\begin{align}
\cos(\cos^{-1}(x)-\cos^{-1}(y))=xy+\sqrt{1-x^2}\sqrt{1-y^2}
\end{align}
the result follows after re-arranging terms.
\end{proof}

\begin{lemma}
\label{lemma:kcluster}
Let $v_1,..,v_k\in\mathbb{R}^d$, $d\geq 2$, $\gamma_1,..,\gamma_k\in\mathbb{R}_{>0}$ and
\begin{align}
\label{eq:sumtozero}
\sum_{i=1}^k \gamma_iv_i=0\in\mathbb{R}^d
\end{align}
then there exists a bi-partitioning of the vectors $c_1\cup c_2 = \{1,..,k\}$ such that

\begin{align}
\max_{i\in c_1, j\in c_2}  \alpha(v_i,v_j) \leq \cos(\frac{\pi}{k-1})
\end{align}
\end{lemma}
\begin{proof}

Lemma \ref{lemma:kcluster} can be equivalently stated as follows:

Let $v_1,..,v_k\in\mathbb{R}^d$, $d\geq 2$, $\gamma_1,..,\gamma_k\in\mathbb{R}_{>0}$ and
\begin{align}
\sum_{i=1}^k \gamma_iv_i=0\in\mathbb{R}^d
\end{align}
then there exists a bi-partitioning of the vectors $c_1\cup c_2 = \{1,..,k\}$ such that

\begin{align}
\min_{i\in c_1, j\in c_2}  \sphericalangle(v_i,v_j) \geq \frac{\pi}{k-1}
\end{align}

As the angle between two vectors is invariant under multiplication with positive scalars $\gamma>0$ we can assume w.l.o.g that $\gamma_i=1$ $i=1,..,k$.  

Let us first consider the case where $d=2$. Let $e_1\in\mathbb{R}^2$ be the first standard basis vector and assume w.l.o.g that the vectors $v_1,..,v_k$ are sorted w.r.t. their angular distance to $e_1$ (they are arranged circular as shows in Figure \ref{fig:proof_bipart}). As all vectors lie in the 2d plane, we know that the sum of the angles between all neighboring vectors has to be equal to $2\pi$.
\begin{align}
\label{eq:constantangle}
\sum_{i=1}^k \sphericalangle(v_i,v_{(i+1)\text{ mod }k}) = 2\pi
\end{align} 
Now let 
\begin{align}
i_1^*=\arg\max_{i\in\{1,..,k\}} \sphericalangle(v_i,v_{(i+1)\text{ mod }k})
\end{align}
and 
\begin{align}
i_2^*=\arg\max_{i\in\{1,..,k\}\setminus i_1^*} \sphericalangle(v_i,v_{(i+1)\text{ mod }k})
\end{align}
be the indices of the largest and second largest angles between neighboring vectors and define the following clusters:
\begin{align}
c_1&=\{i\text{ mod }k|i_1^*<i\leq i_2^*+k[i_2^*<i_1^*]\}\\
c_2&=\{i\text{ mod }k|i_2^*<i\leq i_1^*+k[i_2^*>i_1^*]\}\}
\end{align}
where $[x]=1$ if $x$ is true and $[x]=0$ is $x$ is false.
Then by construction the second largest angle $\sphericalangle(v_{i^*_2},v_{(i^*_2+1) \text{ mod }k})$ minimizes the angle between any two vectors from the two different clusters $c_1$, $c_2$ (see Figure \ref{fig:proof_bipart} for an illustration):
\begin{align}
\min_{i\in c_1, j\in c_2}  \sphericalangle(v_i,v_j) = \sphericalangle(v_{i^*_2},v_{(i^*_2+1) \text{ mod }k})
\end{align}
Hence in $d=2$ we can always find a partitioning $c_1, c_2$ s.t. the minimum angle between any two vectors from different clusters is greater or equal to the 2nd largest angle between neighboring vectors. This means the worst case configuration of vectors is one where the 2nd largest angle between neighboring vectors is minimized.
As the sum of all $k$ angles between neighboring vectors is constant according to \eqref{eq:constantangle}, this is exactly the case when the largest angle between neighboring vectors is maximized and all other $k-1$ angles are equal.

Assume now that the angle between two neighboring vectors is greater than $\pi$. That would mean that there exists a separating line $l$ which passes through the origin and all vectors $v_1,..,v_k$ lie on one side of that line. This however is impossible since $\sum_{l=1}^k v_l=0$. This means that the largest angle between neighboring vectors can not be greater than $\pi$. Hence in the worst-case scenario \begin{align}
\sphericalangle(v_{i^*_2},v_{(i^*_2+1) \text{ mod }k}) \geq \frac{2\pi-\pi}{k-1} = \frac{\pi}{k-1}.
\end{align}
This concludes the proof for $d=2$.

Now consider he case where $d>2$. Let $c_1, c_2$ be a clustering which maximizes the minimum angular distance between any two clients from different clusters. Let
\begin{align}
i^*, j^*=\arg\min_{i\in c_1, j\in c_2}\sphericalangle(v_i,v_j)
\end{align}
then $v_{i^*}$ and $v_{j^*}$ are the two vectors with minimal angular distance. Let $A=[v_{i^*},v_{j^*}]\in\mathbb{R}^{d,2}$ and consider now the projection matrix
\begin{align}
P = A(A^TA)^{-1}A^T
\end{align}
which projects all d-dimensional vectors onto the plane spanned by $v_{i^*}$ and $v_{j^*}$. 
Then be linearity of the projection we have 
\begin{align}
0=P0=P(\sum_{i=1}^k v_i)=\sum_{i=1}^k P(v_i)
\end{align}
Hence the projected vectors also satisfy the condition of the Lemma. As 
\begin{align}
\sphericalangle(Pv_{i^*},Pv_{j^*})=\sphericalangle(v_{i^*},v_{j^*})
\end{align}
and 
\begin{align}
\sphericalangle(Pv_{i},Pv_{j})\geq\sphericalangle(v_{i},v_{j})
\end{align}
for all $i,j\notin\{i^*, j^*\}$ the clustering $c_1, c_2$ is still optimal after projecting and we have found a 2d configuration of vectors satisfying the assumptions of Lemma \ref{lemma:kcluster} with the same minimal cross-cluster angle. In other words, we have reduced the $d>2$ case to the $d=2$ case, for which we have already proven the result. This concludes the proof.

\end{proof}

\begin{figure}
\centering
\includegraphics[width=0.5\textwidth]{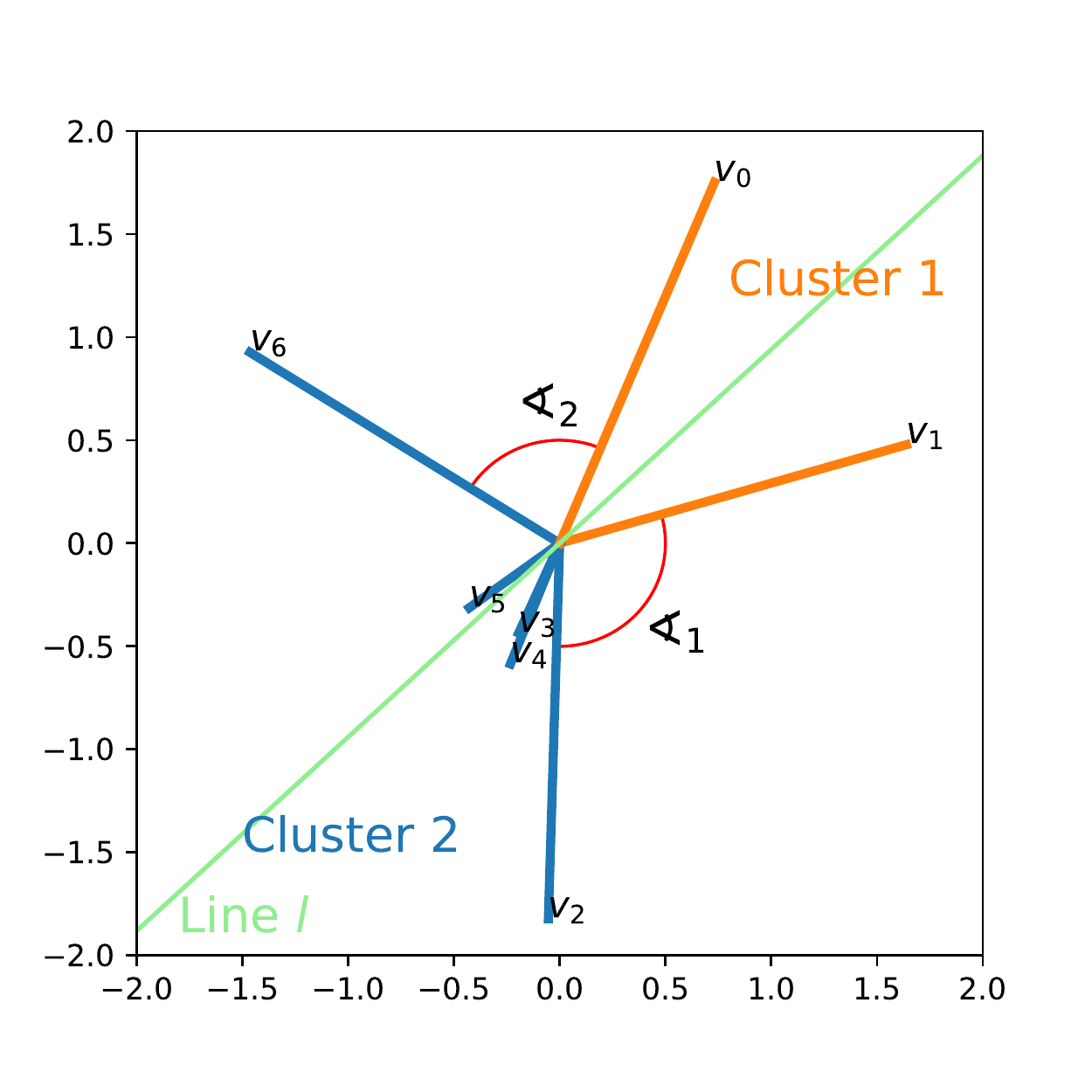}
\caption{Possible configuration in $d=2$. The largest and 2nd largest angle between neighboring vectors (red) separate the two optimal clusters. The largest angle between neighboring vectors is never greater than $\pi$.}
\label{fig:proof_bipart}
\end{figure}
\begin{theorem}[Separation Theorem]
Let $D_1,..,D_m$ be the local training data of $m$ different clients, each dataset sampled from one of $k$ different data generating distributions $\varphi_1,..,\varphi_k$, such that $D_i\sim\varphi_{I(i)}(x,y)$. Let the empirical risk on every client approximate the true risk at every stationary solution of the Federated Learning objective $\theta^*$ s.t.
\begin{align}
\|\nabla R_{I(i)}(\theta^*)\| > \|\nabla R_{I(i)}(\theta^*)-\nabla r_i(\theta^*)\|
\end{align} 
and define 
\begin{align}
\gamma_i:=\frac{\|\nabla R_{I(i)}(\theta^*)-\nabla r_i(\theta^*)\|}{\|\nabla R_{I(i)}(\theta^*)\|}\in[0,1)
\end{align}
Then there exists a bi-partitioning $c_1^*\cup c_2^*=\{1,..,m\}$ of the client population such that that the maximum similarity between the updates from any two clients from different clusters can be bounded from above according to
\begin{align}
&\alpha_{cross}^{max} := \min_{c_1\cup c_2=\{1,..,m\}}\max_{i\in c_1,j\in c_2}\alpha(\nabla r_i(\theta^*), \nabla r_j(\theta^*))\\
&=\max_{i\in c_1^*,j\in c_2^*}\alpha(\nabla r_i(\theta^*), \nabla r_j(\theta^*))\\
&\leq
\begin{cases}
 \cos(\frac{\pi}{k-1})H_{i,j}+\sin(\frac{\pi}{k-1})\sqrt{1-H_{i,j}^2} & \text{ if }H\geq \cos(\frac{\pi}{k-1}) \\
 1 & \text{ else }
\end{cases}
\end{align}
with
\begin{align}
H_{i,j} = -\gamma_i\gamma_j+\sqrt{1-\gamma_i^2}\sqrt{1-\gamma_j^2}\in(-1,1].
\end{align}
At the same time the similarity between updates from clients which share the same data generating distribution can be bounded from below by
\begin{align}
\alpha_{intra}^{min} := \min_{\underset{I(i)=I(j)}{i,j}}\alpha(\nabla_\theta r_i(\theta^*), \nabla_\theta r_j(\theta^*))\geq \min_{\underset{I(i)=I(j)}{i,j}}H_{i,j}.
\end{align}
\end{theorem}

\begin{proof}
For the first result, we know that in every stationary solution of the Federated Learning objective  $\theta^*$ it holds 
\begin{align}
\sum_{l=1}^k \gamma_i\nabla_\theta R_l(\theta^*) = 0
\end{align} 
and hence by Lemma \ref{lemma:kcluster} there exists a bi-partitioning $\hat{c}_1\cup \hat{c}_2 = \{1,..,k\}$ such that 
\begin{align}
\max_{l\in \hat{c}_1, j\in \hat{c}_2}  \alpha(\nabla_{\theta} R_l(\theta^*),\nabla_{\theta} R_j(\theta^*)) \leq \cos(\frac{\pi}{k-1})
\end{align}
Let 
\begin{align}
c_1=\{i|I(i)\in\hat{c}_1, i=1,..,m\}
\end{align}
and 
\begin{align}
c_2=\{i|I(i)\in\hat{c}_2, i=1,..,m\}
\end{align}
and set for some $i\in c_1$ and $j\in c_2$:
\begin{align}
v=\nabla_\theta R_{I(i)}(\theta^*)
\end{align}
\begin{align}
X=\nabla_\theta r_i(\theta^*)-\nabla_\theta R_{I(i)}(\theta^*)
\end{align}
\begin{align}
w=\nabla_\theta R_{I(j)}(\theta^*)
\end{align}
\begin{align}
Y=\nabla_\theta r_j(\theta^*)-\nabla_\theta R_{I(j)}(\theta^*)
\end{align}
Then $\alpha(v,w)\leq\cos(\frac{\pi}{k-1})$ and the result follows directly from Lemma \ref{lemma:different}. 

The second result \eqref{eq:res2} follows directly from Lemma \ref{lemma:same} by setting 
\begin{align}
v=\nabla_\theta R_{I(i)}(\theta^*)
\end{align}
\begin{align}
X=\nabla_\theta r_i(\theta^*)-\nabla_\theta R_{I(i)}(\theta^*)
\end{align}
\begin{align}
Y=\nabla_\theta r_j(\theta^*)-\nabla_\theta R_{I(i)}(\theta^*)
\end{align}
\end{proof}

\end{document}